\documentclass[10pt,twocolumn]{IEEEtran}

\usepackage{times,url}
\usepackage{graphicx}
\usepackage{amsmath,amsthm}
\usepackage{amssymb,empheq}
\usepackage{latexsym}
\graphicspath{{./}{./Fig/}{./Figs/}}
\usepackage{xspace}
\usepackage{url}
\usepackage{multirow}

\usepackage[bf]{caption}
\usepackage[labelfont=bf,font=scriptsize]{subcaption}

\usepackage{algorithm}
\usepackage{algorithmic}

\usepackage[usenames]{color}
\definecolor{shadecolour}{gray}{0.4}

\newcommand{\diag}{\rm{diag}}

\newcommand{\BY}{{\mathbf{Y}}}
\newcommand{\BX}{{\mathbf{X}}}
\newcommand{\BXs}{{\mathbf{Xs}}}

\newcommand{\BD}{{\mathbf{D}}}
\newcommand{\BL}{{\mathbf{L}}}
\newcommand{\Bone}{{\mathbf{1}}}
\newcommand{\BU}{{\mathbf{X}}}

\newcommand{\BB}{{\mathbf{B}}}
\newcommand{\BR}{{\mathbf{X}}}
\newcommand{\BW}{{\mathbf{W}}}

\newcommand{\BZ}{{\mathbf{Z}}}
\newcommand{\BM}{{\mathbf{M}}}

\newcommand{\BT}{{\mathbf{T}}}
\newcommand{\BI}{{\mathbf{I}}}
\newcommand{\T}{{\!\top}}
\newcommand{\Bx}{{\mathbf{x}}}
\newcommand{\By}{{\mathbf{y}}}

\newcommand{\IDH}{{IMH}\xspace}
\newcommand{\IDHs}{{IMHs}\xspace}
\newcommand{\idh}{{Inductive Manifold-Hashing}\xspace}

\renewcommand{\Lambda}{\varLambda}

\newcommand{\st}{{\,\,\mathrm{s.t.\,\,}}}

\newcommand{\trace}{{\mathrm{trace}}}

\newcommand{\w}{{\mathrm{w}}}
\newcommand{\sgn}{{\mathrm{sgn}}}

\ifx\theorem\undefined
\newtheorem{theorem}{Theorem}
\newenvironment{theorem*}{\par\noindent{\bf Theorem\ }}{\hfill\\[2mm]}

\newtheorem{lemma}[theorem]{Lemma}

\newtheorem{corollary}[theorem]{Corollary}
\newenvironment{corollary*}{\par\noindent{\bf Corollary\ }}{\hfill\\[2mm]}
\newtheorem{definition}[theorem]{Definition}

\fi

\newcommand{\RR}{\mathbb{R}}
\newcommand{\NN}{\mathbb{N}}

\DeclareMathOperator{\Ncal}{\mathcal{N}}
\newcommand{\Bc}{{\mathbf{c}}}
\def\Var{{\rm Var}\,}
\def\E{{\rm E}\,}

\makeatletter
\DeclareRobustCommand\onedot{\futurelet\@let@token\@onedot}
\def\@onedot{\ifx\@let@token.\else.\null\fi\xspace}

\def\eg{\emph{e.g}\onedot} 
\def\ie{\emph{i.e}\onedot}

\makeatother

\begin{document}

\title{Hashing on Nonlinear Manifolds}

\author{
         Fumin Shen,
         Chunhua Shen,
         Qinfeng Shi,
         Anton van den Hengel,
         Zhenmin Tang,
         Heng Tao Shen

\thanks
{
    F. Shen  is with School of Computer Science and Engineering, University of Electronic Science and Technology of China, Chengdu 611731,  China
    (e-mail: fumin.shen@gmail.com).
   Part of this work was done
   when the first author was visiting The University of Adelaide.
 }
\thanks
{
    C. Shen, Q. Shi and A. van den Hengel are with  School
    of Computer Science at The University of Adelaide, SA 5005, Australia (e-mail: \{chunhua.shen,
    javen.shi, anton.vandenhengel\}@adelaide.edu.au). C. Shen and A. van den Hengel are also with Australian Centre for Robotic Vision.
      Correspondence should be addressed to C. Shen.
}
 \thanks
 {Z. Tang is with the School of Computer Science and Technology, Nanjing University of Science and Technology, Nanjing 210094, P.R. China (e-mail: tzm.cs@mail.njust.edu.cn).
 }
 \thanks
 { H. T. Shen is with School of Information Technology and Electrical Engineering, The University of Queensland, Australia
  (E-mail: shenht@itee.uq.edu.au).
 }
}

\maketitle

\begin{abstract}

    Learning based hashing methods have attracted considerable
    attention due to their ability to greatly increase the scale at
    which existing algorithms may operate.
    Most of these methods are designed
    to generate binary codes
     preserving the Euclidean similarity in the original space.
    Manifold learning techniques, in contrast,
    are better able to model the intrinsic structure
    embedded in the original high-dimensional data.
    The complexities of these models, and the problems with
    out-of-sample data, have previously rendered them unsuitable for
    application to large-scale embedding, however.

    In this work, how to learn compact binary embeddings
    on their  intrinsic manifolds is considered.
    In order to address the above-mentioned difficulties,
    an efficient, inductive solution to the
    out-of-sample data problem, and a process by which
    non-parametric manifold learning may be used as the basis of a hashing method is proposed.
    The proposed approach thus allows the development of a range of new hashing techniques
    exploiting the flexibility of the wide variety of manifold learning approaches available.
    It is particularly shown that hashing on the basis of t-SNE
    \cite{tSNE2008}
    outperforms state-of-the-art hashing methods on
    large-scale benchmark datasets, and %
    is very effective for image classification with very short code
    lengths.
The proposed hashing framework is shown to be easily improved, for example, by minimizing the quantization error with learned orthogonal rotations. In addition, a supervised inductive manifold hashing framework is developed by incorporating the label information, which is shown to greatly advance the semantic retrieval performance.

\end{abstract}


\section{Introduction}

One key challenge in many large scale image data based applications is how to index and organize the  data accurately, but also efficiently.
    Various hashing techniques have attracted considerable attention
    in computer vision, information retrieval
    and machine learning \cite{LSH99,PCA-ITQ,ICML13SHEN,CVPR14Lin,SSH2012,SH08},
    and seem to offer great promise towards this goal.
    Hashing methods aim to encode documents or images
    as a set of short binary codes, while maintaining aspects of the structure
    of the original data (\eg, similarities between data points).
    The advantage of these compact binary representations is that
    pairwise comparisons may be carried out extremely efficiently in the Hamming space.
    This means that many algorithms which are based on such pairwise
    comparisons can be made more efficient, and applied to much larger datasets.
Due to the flexibility of hash codes, hashing techniques can be applied in many ways.
    one can, for example, efficiently
    perform similarity search by exploring only those data points
    falling into the close-by buckets to the query by the Hamming
    distance, or use the binary representations for other tasks like
    image classification.

Locality sensitive hashing (LSH) \cite{LSH99}  is one of the most well-known
{\it data-independent} hashing methods, and generates  hash codes based on
random projections.
With the success of LSH, random hash functions have been extended to
 several  similarity measures, including $p$-norm distances
 \cite{LSH_p}, the Mahalanobis metric \cite{kulis2009fast}, and kernel
 similarity \cite{KLSH2009,raginsky2009locality}. However, the methods belonging to the LSH
 family normally require relatively long hash codes and several hash
 tables to achieve both high precision and recall.  This leads to a
 larger storage cost than would otherwise be necessary, and thus limits
 the sale at which the algorithm may be applied.

{\it Data-dependent} or learning-based hashing methods have
been developed with the goal of learning more compact hash codes.
Directly learning binary embeddings typically results in an optimization
problem which is very difficult to solve, however.
Relaxation is often used to simplify the optimization (\eg,
\cite{LDAhash2012,SSH2012}).
As in LSH, the methods aim to identify a set of hyperplanes,
but now these hyperplanes are learned, rather than randomly selected.
For example, PCAH \cite{SSH2012}, SSH \cite{SSH2012}, and iterative quantization (ITQ)
\cite{PCA-ITQ} generate linear hash functions through simple principal component analysis (PCA) projections, while LDAhash \cite{LDAhash2012} is based on Linear Discriminant Analysis (LDA).
Extending this idea, there are also methods which learn hash functions
in a kernel space, such as binary  reconstructive embeddings
(BRE) \cite{BRE2009}, random maximum margin hashing (RMMH)
\cite{RMMH2011} and kernel-based supervised hashing (KSH)
\cite{KSH2012}.  In a  departure from such methods,  however, spectral
hashing (SH) \cite{SH08}, one of the most popular learning-based
methods, generates hash codes by
    solving the relaxed mathematical program that is similar to the
    one in Laplacian eigenmaps \cite{LE2001}.

\setcounter{figure}{1}

\begin{figure*}
\centering
\begin{align*}
\begin{subfigure}{0.08\textwidth}
\centering
\includegraphics[height = 1.3cm]{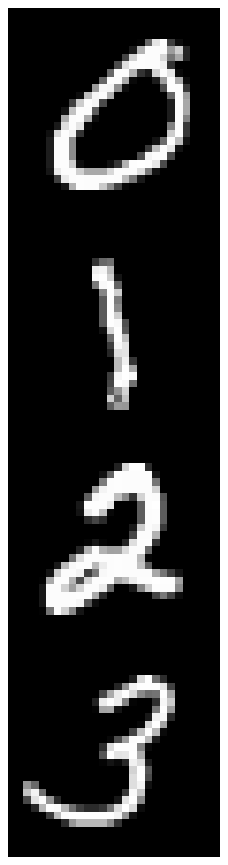}
\caption{Queries}
\end{subfigure} %
\begin{subfigure}{0.21\textwidth}
\includegraphics[height = 1.3cm]{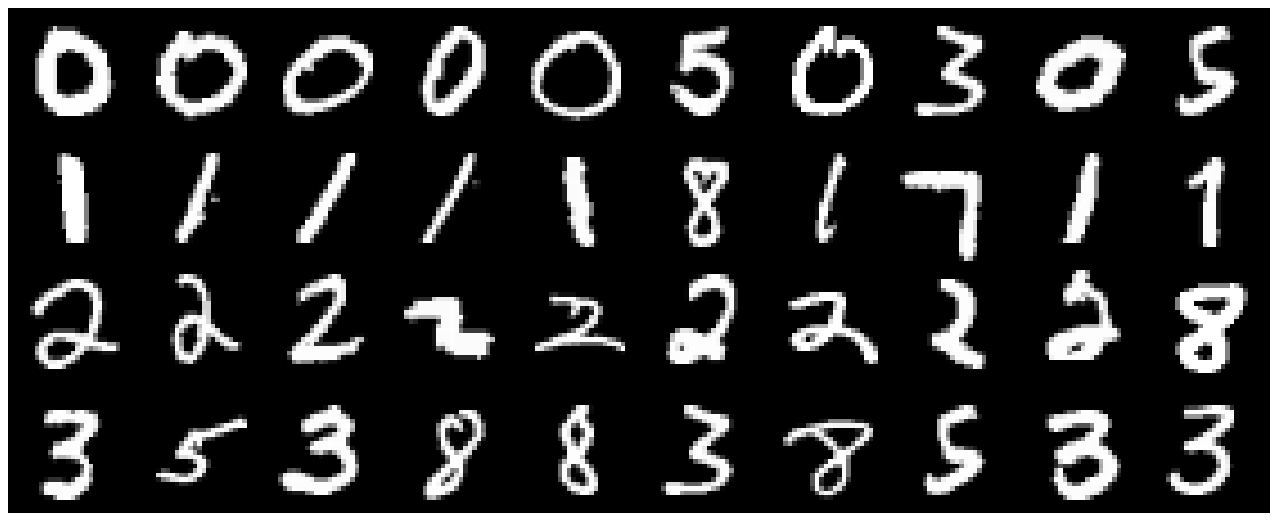}
\caption{$ \ell_2$ dist.\ on 784D}
\end{subfigure} %
\begin{subfigure}{0.21\textwidth}
\includegraphics[height = 1.3cm]{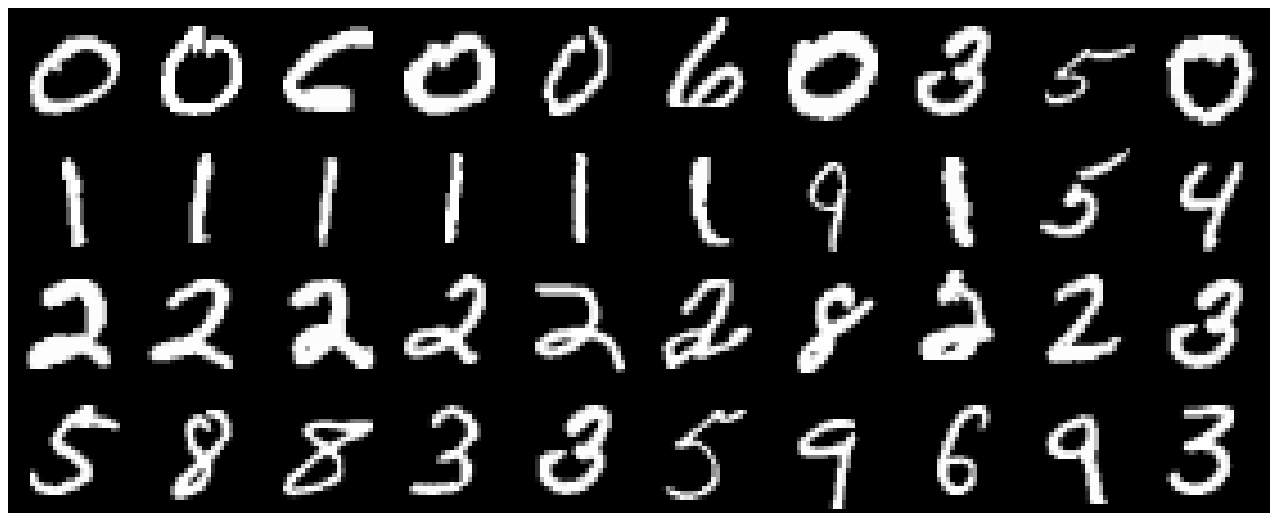}
\caption{LSH with 128-bits}
\end{subfigure}
\begin{subfigure}{0.21\textwidth}
\includegraphics[height = 1.3cm]{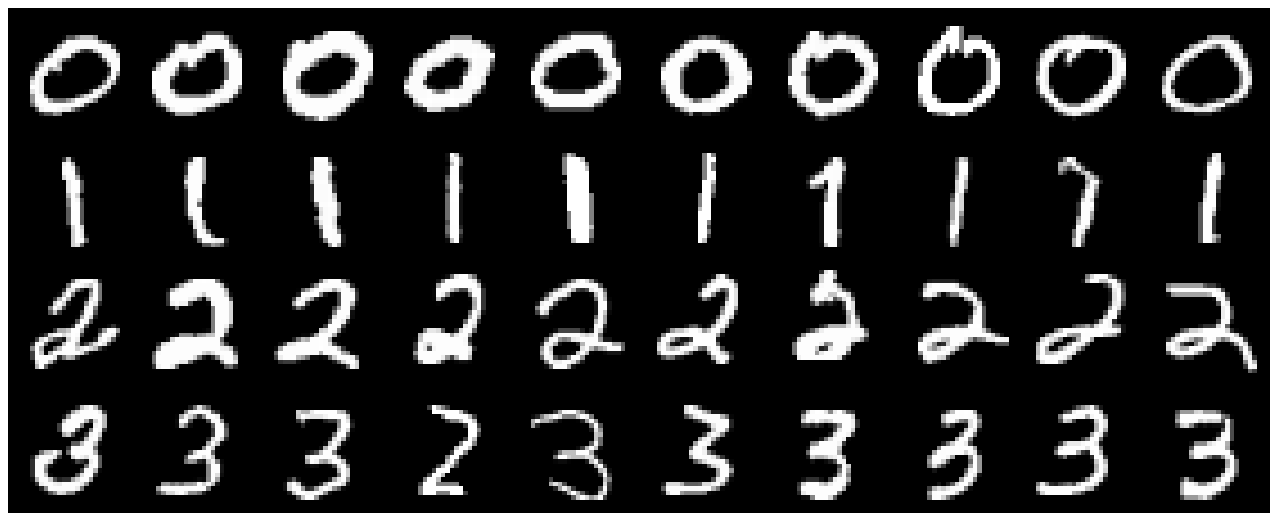}
\caption{$ \ell_2$ dist.\ on embeded 48D}
\end{subfigure} %
\begin{subfigure}{0.21\textwidth}
\includegraphics[height = 1.3cm]{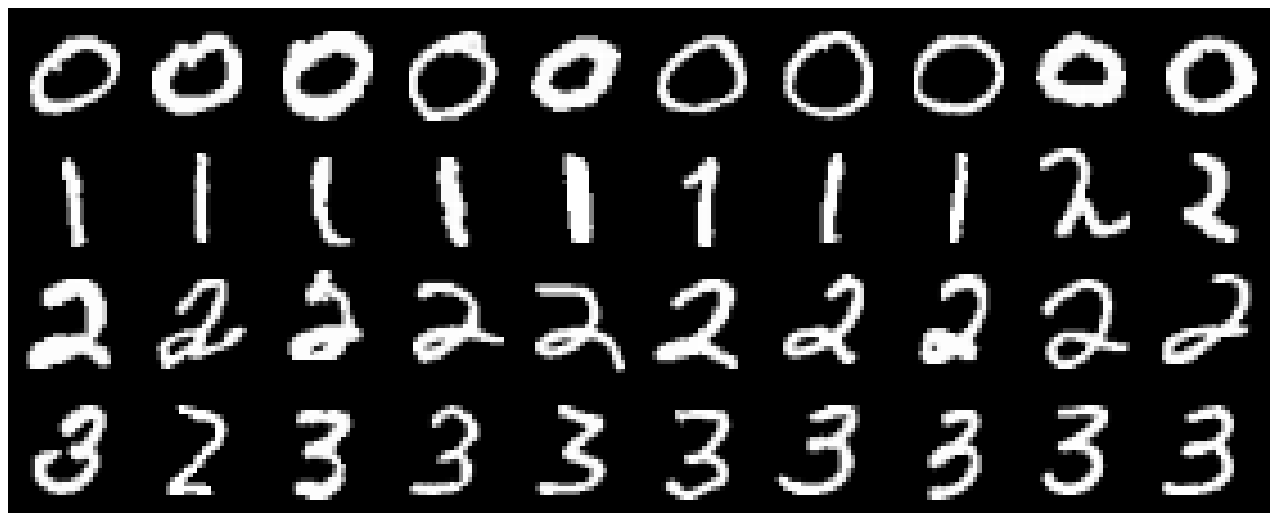}
\caption{Hamm. dist.\ with 48-bits}
\end{subfigure}
\end{align*}
\caption{Top 10 retrieved digits for 4 queries (a) on a subset of MNIST with  300 samples. Search is conducted in the original feature space (b, c) and nonlinear embedding space by t-SNE \cite{tSNE2008}
(d, e) using Euclidean distance (b, d) and Hamming distance (c, e). }
\label{query_digit}
\end{figure*}

Embedding the original data into a low dimensional space while
simultaneously preserving the inherent neighborhood structure is
critical for learning compact, effective hash codes.
In general, nonlinear manifold learning  methods are more
powerful than  linear dimensionality reduction techniques, as they are
able to more effectively preserve the \textit{local} structure of the
input data without assuming \textit{global} linearity
\cite{Kumar2008}.
The geodesic distance on a manifold has been shown to outperform the Euclidean distance in the high-dimensional  space for image retrieval \cite{He_manifold_2004}, for example.
Figure~\ref{query_digit} demonstrates that searching using
either the Euclidean or Hamming distance after nonlinear embedding results in more semantically accurate
neighbors than the same search in the original feature space, and thus
that low-dimensional embedding may actually improve retrieval or
classification performance.
However,
the only widely used nonlinear embedding method for hashing is Laplacian eigenmaps (LE)
(\eg, in \cite{SH08,STH2010,AGH2011}).
Other effective manifold learning approaches (\eg, Locally Linear Embedding (LLE) \cite{LLE2000},
Elastic Embedding (EE) \cite{EE2010} or t-Distributed Stochastic Neighbor Embedding (t-SNE) \cite{tSNE2008}) have rarely been explored for hashing. Very recently, the authors of \cite{Irie_2014_CVPR} choose to jointly minimize the LLE embedding error and the quantization loss with an orthogonal rotation.

One problem hindering the use of manifold learning for hashing is that
these methods do not directly scale to large datasets. For example, to
construct the neighborhood graph (or pairwise similarity matrix) in
these algorithms for $n$ data points is $\mathit{O}(n^2)$ in time,
which is intractable for large datasets.
The second problem is that they are typically non-parametric and thus
cannot efficiently solve the critical out-of-sample extension problem.
This fundamentally limits their application to hashing, as generating
codes for new samples is an essential part of the problem.   One of
the widely used solutions for the methods involving spectral
decomposition (\eg, LLE, LE and  isometric feature mapping (ISOMap) \cite{ISOMAP2000}) is the
Nystr\"om extension \cite{Nystrom2004,Kumar2008}, which solves the problem by
learning  eigenfunctions of a kernel matrix. As mentioned in
\cite{SH08}, however, this is impractical for large-scale hashing
since the Nystr\"om extension is as expensive as doing exhaustive
nearest neighbor search ($\mathit{O}(n)$). A more significant problem,
however,  is the fact that the Nystr\"om extension cannot be directly
applied to {\it non-spectral} manifold learning methods such as t-SNE.

In order to address the out-of-sample extension problem, this study proposes a
new non-parametric regression approach which is both efficient and
effective.  This method allows rapid assignment of new codes to
previously unseen data in a manner which preserves the underlying
structure of the manifold.  Having solved the out-of-sample extension
problem, a method by which a learned manifold may be used
as the basis for a binary encoding is introduced.  This method is designed so as to
generate encodings which reflect the geodesic distances along such
manifolds.  On this basis,  a range of new embedding
approaches based on a variety of manifold learning methods are developed.  The best
performing of these is based on manifolds identified through t-SNE,
which has been shown to  be effective in discovering semantic
manifolds amongst the set of all images \cite{tSNE2008}.

Given the computational complexity of many manifold learning methods, it is shown in this work
 that it is possible to learn the manifold on the basis of a small subset of the data $\BB$ (with size $m \ll n$), and subsequently to inductively insert the remainder of the data, and any out-of-sample data, into the embedding in $\mathit{O}(m)$ time per point.
This process leads to an embedding method  labelled as \idh (\IDH) which is shown to
outperform state-of-the-art methods on several large scale datasets both quantitatively and qualitatively.

As an extension, this study shows that the proposed \IDH framework can be naturally improved by minimizing the quantization error, for example, through orthogonal rotations. Significant performance gains are achieved by this simple step, as shown in Section \ref{Sec:ITQ}. Based on supervised subspace learning, this study also presents a supervised inductive manifold hashing framework (IMHs) (Section \ref{Sec:supervised}), which is shown to significantly advance the semantic retrieval performance of \IDH.

The rest of this paper is organized as follows. In Section \ref{related}, some representative hashing methods related are briefly reviewed. Section \ref{Sec:IMH} describes the proposed \idh framework, followed by the experimental results in Section \ref{Sec:exp}. The \IDH method is further shown  to be improved by learned orthogonal rotations in Section \ref{Sec:ITQ}. Section \ref{Sec:supervised} introduces the supervised  extension of the inductive manifold hashing framework based on supervised subspace learning.

This paper is an extended version of the work previously published in  \cite{CVPR13aShen}. Major improvements over
\cite{CVPR13aShen} include the quantization error minimization by learned rotations (Section \ref{Sec:ITQ}) and the supervised extension (Section \ref{Sec:supervised}).

\section{Related work}
\label{related}
Learning based or data-dependent hashing has achieved considerable attention recently in computer vision, machine learning and information retrieval community. Many hashing methods have been proposed by applying different learning algorithms, including unsupervised methods \cite{PCA-ITQ,AGH2011,SH08,STH2010}, and (semi-)supervised methods \cite{BRE2009}\cite{KSH2012}\cite{CVPR14Lin}\cite{SSH2012}.
       In this section,   some representative unsupervised hashing methods related to the proposed method are briefly reviewed.
\vspace{-0.4cm}
\subsection{Spectral Hashing}
Weiss et al. \cite{SH08} formulated the spectral hashing (SH) problem as
\begin{align}
\label{EQ:SH}
\min_\BY & \sum_{\Bx_i,\Bx_j \in \BU} \w(\Bx_i, \Bx_j) \|\By_i -\By_j\|^2
\\
\st & \BY \in \{-1,1\}^{n \times r}\notag, \,
  \BY^{\T} \BY = n\BI \notag, \,
  \BY^\T \Bone = 0. \notag
\end{align}
Here $\By_i \in \{-1, 1\}^r$, the $i$th row in $\BY$, is the hash code one wants to learn for $\Bx_i \in \mathbf{R}^d$, which is one of the $n$ data points in the training data set $\BU$.
$\BW \in \mathbf{R}^{n \times n}$ with $\BW_{ij} = \w(\Bx_i, \Bx_j) = \exp(-\|\Bx_i - \Bx_j\|^2/\sigma^2)$ is the graph affinity matrix, where $\sigma$ is the bandwidth parameter.
 $\BI$ is  the identity matrix.
The last two constraints force the learned hash bits to be uncorrelated and balanced, respectively.
By removing the first constraint (\ie, \textit{spectral relaxation} \cite{SH08}),  $\BY$ can be easily obtained by
spectral decomposition on the Laplcaian matrix $\BL =
 \BD - \BW$, where $\BD = \diag(\BW \Bone)$ and $\Bone$ is the vector with all ones.
 However, constructing $\BW$ is $\mathit{O}(dn^2)$ (in time) and calculating the Nystr\"om extension for a new point is $\mathit{O}(rn)$, which are both intractable for large datasets.
It is assumed in SH \cite{SH08}, therefore, that the data are sampled
from a uniform distribution, which leads to a simple analytical
eigenfunction solution of 1-D Laplacians. However, this strong
assumption is not true in practice and the manifold structure of the
original data are thus destroyed \cite{AGH2011}.

 SH was extended into the tangent space in \cite{chaudhry2010fast},
 however, based on the same uniform assumption.
 \cite{chaudhry2010fast} also proposed a non-Euclidean SH algorithm
 based on nonlinear clustering, which is $\mathit{O}(n^3)$ for
 training and $\mathit{O}(m + n/m)$ for testing.
Weiss et al. \cite{MDSH} then improved SH by expanding the codes  to include the outer-product eigenfunctions instead of  only single-dimension eigenfunctions in SH, \ie, multidimensional spectral hashing (MDSH).

\subsection{Graph based hashing}
To efficiently solve problem \eqref{EQ:SH}, anchor graph hashing (AGH)  \cite{AGH2011} approximated the affinity matrix $\BW$ by the low-rank matrix $\hat{\BW} = \BZ\Lambda^{-1}\BZ$, where $\BZ \in \mathbf{R}^{n \times m}$ is the
normalized affinity matrix (with $k$ non-zeros in each row) between the training samples and $m$ {\textit{anchors}}
(generated by $K$-means), and $\Lambda^{-1}$ normalizes $\hat{\BW}$ to be doubly stochastic.
Then  the desired hash functions may be efficiently identified by binarizing
the Nystr\"om eigenfunctions \cite{Nystrom2004} with the approximated affinity matrix $\hat{\BW}$.  AGH
is thus efficient, in that it has linear training time and constant search time, but
as is the case for SH \cite{SH08}, the generalized eigenfunction is derived only  for the Laplacian eigenmaps embedding.

Different from SH and AGH, Locally Linear Hashing (LLH \cite{Irie_2014_CVPR}) constructs the graph affinity by locality-sensitive sparse coding to better capture the local linearity of manifolds. With the obtained affinity matrix, LLH formulates hashing as a joint optimization problem of LLE embedding error and quantization loss.

\subsection{Self-Taught Hashing}
Self-taught hashing (STH) \cite{STH2010} addressed the out-of-sample problem by a novel way:  hash functions are obtained by
training a support vector machine (SVM) classifier for each bit using the pre-learned binary codes as class labels. The binary codes were learned by directly solving \eqref{EQ:SH} with a cosine similarity function.
This process has  prohibitive computational and memory costs, however,
and training the SVM can be very time consuming for dense data.
Very recently, this idea was applied in collaboration with graph cuts to the binary coding problems to bypass continuous relaxation \cite{GCC_ECCV2014}.

\section{The proposed method}
\label{Sec:IMH}
\subsection{Inductive learning for hashing}
Assuming that one has the  manifold-based embedding $\BY := \{\By_1,
\By_2, $ $ \cdots, $ $ \By_n\}$ for the entire training data $\BX:=
\{\Bx_1, $ $ \Bx_2, $ $\cdots, $ $\Bx_n\}$.
Given a new data point $\Bx_q$, one aims to
generate an
embedding $\By_q$ which preserves the local neighborhood relationships
among its neighbors $\mathcal{N}_k(\Bx_q)$ in $\BX$.
The following simple objective is utilized:
\begin{equation}
\label{EQ:newdata}
\mathcal{C}(\By_q) = \sum_{i = 1}^n \w(\Bx_q, \Bx_i) \|\By_q -\By_i\|^2,
\end{equation}
where
\begin{align*}
\w(\Bx_q, \Bx_i) =
\left\{
\begin{array}{cl}
\exp (-\|\Bx_q - \Bx_i\|^2/\sigma^2), & \text{if } \Bx_i \in \mathcal{N}_k(\Bx_q),\\
0 & \text{otherwise}.
\end{array}\right.
\end{align*}
Minimizing \eqref{EQ:newdata} naturally
uncovers an embedding for the new point on the basis of its nearest neighbors on the
low-dimensional manifold initially learned on the base set.
That is, in the low-dimensional space, the new embedded location for the point should be close to those of the points close to it in the original space.

Differentiating $\mathcal{C}(\By_q)$ with respect to $\By_q$,  one obtains
\begin{equation}
\frac{\partial \mathcal{C}(\By_q)}{\By_q}\bigg|_{\By_q = \By_q^{\star}} = \,2\sum_{i = 1}^n \w(\Bx_q, \Bx_i) (\By_q^{\star} -\By_i) = 0,
\end{equation}
which leads to the optimal solution
\begin{equation}
\label{EQ:Induction}
\By_q^{\star} = \frac{\sum_{i = 1}^n \w(\Bx_q, \Bx_i) \By_i}{\sum_{i = 1}^n \w(\Bx_q, \Bx_i)}.
\end{equation}
Equation \eqref{EQ:Induction} provides a simple inductive formulation for the embedding: produce the embedding for a new data point by a (sparse) locally linear combination of the base embeddings.

The proposed approach here is inspired by
        Delalleau et al.\  \cite{Olivier2005}, where they have focused
        on non-parametric graph-based learning in  semi-supervised
        classification. The aim of this study here is completely different: the present work tries
        to scale up the manifold learning process for hashing in an
        unsupervised manner.

 The resulting solution \eqref{EQ:Induction} is consistent with the basic
smoothness
 assumption in manifold learning, that close-by data points lie on or close to a locally linear manifold \cite{LLE2000,ISOMAP2000,LE2001}.
 This local-linearity  assumption has also been widely used in semi-supervised learning \cite{Olivier2005,LLC2009},  image coding \cite{wang2010locality}, and similar.
This paper proposes to apply this assumption to hash function learning.

However, as aforementioned,  \eqref{EQ:Induction} does not scale well
for both computing $\BY$ ($\mathit{O}(n^2)$ \eg, for LE) and
out-of-sample extension ($\mathit{O}(n)$), which is intractable for
large scale tasks.
Next, it is  shown that the following prototype algorithm is able to approximate $\By_q$ using only a small base set well.

\subsection{The prototype algorithm}
This prototype algorithm is based on entropy numbers defined below.
\begin{definition} [Entropy numbers\cite{HerWil02c}] Given any $Y \subseteq \RR^r$ and $m \in \NN$, the $m$-th entropy number $\epsilon_m(Y)$ of $Y$ is defined as
\[\epsilon_{m}(Y) := \inf\{\epsilon>0| \Ncal (\epsilon, Y, \|\cdot - \cdot \|) \leq m \},\]
where $\Ncal$ is the covering number. This means $\epsilon_{m}(Y)$ is the smallest radius that $Y$ can be covered by less or equal to $m$ balls.
\end{definition}

Inspired by Theorem 27 of \cite{HerWil02c},  a prototype algorithm is constructed below. One can use $m$ balls to cover $Y$, thus obtain $m$ disjoint nonempty subsets $Y_1, Y_2, \cdots, Y_m$ such that for any $\epsilon >\epsilon_{m}(Y) $, $\forall j \in \{1,\cdots,m\}, \exists \Bc_j \in \RR^r, \st \forall \By \in Y_j, \|\Bc_j - \By\| \leq \epsilon$ and $\bigcup_{j=1}^m Y_j = Y$. One can see that each $Y_j$ naturally forms a cluster with the center $\Bc_j$ and the index set $I_j = \{i|\By_i \in Y_j\}$.

 Let $\alpha_i =\frac{ \w(\Bx_q, \Bx_i) }{\sum_{j =1}^n \w(\Bx_q, \Bx_j)} $ and $C_j = \sum_{i\in I_j} \alpha_i$.  For each cluster index set $I_j, j = 1, \cdots, m$,  $\ell_j = \lfloor mC_j+1 \rfloor$ many indices are randomly drawn from $I_j$ proportional to their weight $\alpha_i$. That is, for $\mu \in \{1,\cdots, \ell_j\}$,  the $\mu$-th randomly drawn index  $u_{j,\mu}$,
 \[ \Pr ( u_{j,\mu} = i) = \frac{\alpha_i}{C_j}, \forall j \in \{1,\cdots,m\}. \]
  $\hat{\By}_q$ is constructed as
 \begin{align}
 \label{eq:infer-sample}
 \hat{\By}_q = \sum_{j=1}^m \frac{C_j}{\ell_j} \sum_{\mu=1}^{\ell_j}\By_{u_{j,\mu}}.
 \end{align}

 \begin{lemma} There is at most $2m$ many unique $\By_{u_{j,\mu}}$ in $\hat{\By}_q$.
 \end{lemma}
 \begin{proof}
 $\sum_{j=1}^m \ell_j \leq \sum_{j=1}^m  (mC_j+1) = 2m$.
 \end{proof}

  The following lemma shows that through the prototype algorithm the mean is preserved and variance is small.
  \begin{lemma}
  \label{lem:mean-var}
   The following holds
  \begin{align}
  \E[\hat{\By}_q] ={\By}_q, \Var(\hat{\By}_q) \leq \frac{\epsilon^2}{m}.
  \end{align}
 \end{lemma}
  \begin{proof}
  \begin{align*}
 &\E[\hat{\By}_q] = \E[ \sum_{j=1}^m \frac{C_j}{\ell_j} \sum_{\mu=1}^{\ell_j}\By_{u_{j,\mu}}]= \sum_{j=1}^m \frac{C_j}{\ell_j} \sum_{\mu=1}^{\ell_j}\E[\By_{u_{j,\mu}}]\\
  & =\sum_{j=1}^m \frac{C_j}{\ell_j} \sum_{\mu=1}^{\ell_j} \sum_{i \in I_j} \frac{\alpha_i}{C_j}\By_i =\sum_{j=1}^m \sum_{i \in I_j} \alpha_i\By_i = {\By}_q.\\
 & \Var(\hat{\By}_q) = \sum_{j=1}^m \sum_{\mu=1}^{\ell_j}\Var(\frac{C_j}{\ell_j} \By_{u_{j,\mu}}) \leq \sum_{j=1}^m \frac{C_j^2}{\ell_j^2} \sum_{\mu=1}^{\ell_j}\epsilon^2\\
 & = \sum_{j=1}^m \frac{C_j^2}{\ell_j} \epsilon^2 \leq \sum_{j=1}^m \frac{C_j^2 }{mC_j} \epsilon^2 = \frac{ \sum_{j=1}^m C_j^2 }{m} \epsilon^2 =  \frac{ \epsilon^2}{m}.
 \end{align*}
 \end{proof}

   \begin{theorem}
   \label{thm:appox}
   For any even number $n' \leq n$. If Prototype Algorithm uses $n'$ many non-zero $\By \in Y$ to express $\hat{\By}_q$, then
  \begin{align}
  \Pr[\|\hat{\By}_q - {\By}_q\| \ge t ] < \frac{2(\epsilon_{\frac{n'}{2}}(Y))^2}{n't^2}.
  \end{align}
 \end{theorem}
  \begin{proof}
 Via Chebyshev's inequality and Lemma~\ref{lem:mean-var}, one gets
   \begin{align*}
   \Pr\Big( \|\hat{\By}_q - {\By}_q\| \ge k \sqrt{\Var(\hat{\By}_q)} \Big) \leq \frac{1}{k^2}.
    \end{align*} Let $t = k \sqrt{\Var(\hat{\By}_q)}$ and $\epsilon \to \epsilon_{\frac{n'}{2}}(Y)$ yields the theorem.
 \end{proof}

\begin{corollary} For an even number $n'$, any $\epsilon > \epsilon_{\frac{n'}{2}}(Y)$, any $\delta \in (0,1)$ and any $t >0$, if $n' \ge  \frac{2\epsilon^2}{\delta t^2}$, then with probability at least $1-\delta$,
\[\|\hat{\By}_q - {\By}_q\| < t.\]
\end{corollary}
  \begin{proof} Via Theorem~\ref{thm:appox},  for $\epsilon > \epsilon_{\frac{n'}{2}}(Y)$,
  $ \Pr[\|\hat{\By}_q - {\By}_q\| \ge t ] < \frac{2\epsilon^2}{n't^2}.$ Let
   $   \delta \ge \frac{2\epsilon^2}{n't^2}$,
       the following holds $  n' \ge \frac{2\epsilon^2}{ \delta t^2}$.
 \end{proof}

  The quality of the approximation depends on $\epsilon_{\frac{n'}{2}}(Y)$ and $n'$. If data has strong clustering pattern, \ie data within each cluster are very close to cluster center, one will have small $\epsilon_{\frac{n'}{2}}(Y)$, hence better approximation. Likewise, the bigger $n'$ is, the better approximation is.

\subsection{Approximation of the prototype algorithm}
For a query point $\Bx_q$, the prototype algorithm samples from clusters and then construct $\hat{\By}_q$. The clusters can be obtained via clustering algorithm such as K-means. For each cluster, the higher $C_j = \sum_{i\in I_j} \alpha_i$, the more draws are made. At least one draw is made from each cluster. Since the $n$ could be potentially massive, it is impractical to rank (or compute and keep a few top ones) $\alpha_i$ with in each cluster. Moreover, $\w(\Bx_q, \Bx_j)$ depends on $\Bx_q$ --- for a different query point $\Bx_{q'}$, $\w(\Bx_{q'}, \Bx_i)$ may be very smaller even if $\w(\Bx_q, \Bx_i) $ is high. Thus one needs to consider the entire $X$ instead of a single $\Bx_q$.

Let $\alpha_i(\Bx_q) = \frac{ \w(\Bx_q, \Bx_i) }{\sum_{j =1}^n \w(\Bx_q, \Bx_j)}$. Ideally, for each cluster, one wants to select the $\By_i$ that has high overall weight $O_i = \sum_{\Bx_q \in X} \alpha_i(\Bx_q)$. For large scale $X$, the reality is that one does not have access to $\w(\Bx, \Bx')$ for all $\Bx, \Bx' \in X$. Only limited information is available such as cluster centers $\{\Bc_j, j \in \{1, \cdots, m\}\}$ and $\w(\Bc_j, \Bx), \Bx \in X $. Fortunately, the clustering result gives useful information about $O_i$.
The cluster centers $\{\Bc_j, j \in \{1, \cdots, m\}\}$ have the largest overall weight w.r.t the points from their own cluster, \ie $\sum_{{i} \in I_j} { \w({\Bc_j}, \Bx_i) }$.
This suggests one should select all cluster centers to express $\hat{\By}_q$.
 For a base set $B$, and any query point $\Bx_q$,
the embedding is predicted as
\begin{align}
\label{EQ:infer-anchor}
\hat{\By}_q =  \frac{\sum_{\Bx \in B} \w(\Bx_q, \Bx) \By}{\sum_{\Bx \in B} \w(\Bx_q, \Bx)}.
\end{align}

Following many methods in the area (\eg, \cite{SH08,AGH2011}), the general inductive hash function is formulated by  binarizing the low-dimensional embedding
\begin{equation}
\label{EQ:hash_fun}
h(\Bx) = \sgn\left( \frac{\sum_{j = 1}^m \w(\Bx,\Bc_j) \By_j}{\sum_{j = 1}^m \w(\Bx,\Bc_j)}\right),
\end{equation}
where $\sgn(\cdot)$ is the sign function and $\BY_{\BB}:= $ $\{\By_1,
$ $\By_2, $ $\cdots, $ $ \By_m\}$ is the embedding for
the base set $\BB $ $ := $ $ \{\Bc_1, $ $ \Bc_2, $ $ \cdots, \Bc_m\}$,
which is the cluster centers obtained by K-means.
Here the embedding $ \By_i $ are assumed to be centered on the
origin.
The proposed hashing method is termed
 \textit{\idh} (\IDH).
The inductive hash function provides a natural means for
generalization to new data, which has a constant $\mathit{O}(dm + rk)$
time.
With this, the embedding for the training data becomes %
\begin{equation}
\BY = \Bar{\BW}_{\BU\BB} \BY_{\BB},
\label{EQ:SubInduction}
\end{equation}
where $\bar{\BW}_{\BU\BB}$ is defined such that $\bar{\BW}_{ij} $ $ =
$ $\frac{\w(\Bx_i, \Bc_j)}{\sum_{i = 1}^m \w(\Bx_i,\Bc_j)}, $ $ \text{for }
\Bx_i $ $\in \BU, $ $ \Bc_j \in \BB$.

Although the objective function \eqref{EQ:newdata} is formally related to LE, it is general in preserving local similarity.
The embeddings $\BY_{\BB}$ can be learned by any appropriate
 manifold learning method which preserves the similarity of interest in the low dimensional space.
Several other embedding methods are empirically evaluated in
Section~\ref{Embedding_selection}. Actually, as will be shown, some manifold learning methods
(\eg,
t-SNE described in Section \ref{SEC:tSNE}) can be better choices for learning binary codes, although LE has been widely used.
Two methods for learning $\BY_{\BB}$ will be discussed in the sequel.

 The \idh framework is summarized in Algorithm 1. Note that the  computational cost
is dominated by K-means in the first step, which is $\mathit{O(dmnl)}$ in time (with $l$ the number of iterations). Considering that $m$ (normally a few hundreds) is much less than $n$,
and is a function of manifold complexity rather than the volume of data,
the total training time is linear in the size of training set.
If the embedding method is LE, for example, then using \IDH to compute $\BY_{\BB}$ requires constructing
 the small affinity matrix $\BW_{\BB}$ and solving $r$ eigenvectors of the $m \times m$ Laplacian matrix $\BL_{\BB}$
which is $\mathit{O}(dm^2 + rm)$.
Note that in step 3, to compute $\Bar{\BW}_{\BR\BB}$, one needs to compute the distance matrix between $\BB$ and $\BR$, which is a natural output of K-means, or can be computed additionally in $\mathit{O}(dmn)$ time.
The training process on a dataset of 70K items with 784 dimensions can thus be achieved in a few seconds on a standard desktop PC.

\begin{algorithm}[t!]
\caption{\small \idh (\IDH)}
{
\footnotesize
\textbf{Input: }
 Training data $\BU := \{\Bx_1, \Bx_2, \ldots, \Bx_n\}$, code length $r$, base set size $m$,
 neighborhood size $k$

\textbf{Output: } Binary codes $\BY := \{\By_1, \By_2, \ldots, \By_n\} \in \mathbb{B}^{n \times r}$

1) Generate the base set $\BB$ by random sampling or clustering (\eg K-means).

2) Embed $\BB$ into the low dimensional space by \eqref{EQ:tSNE}, \eqref{EQ:Obj_main}  or any other
appropriate manifold leaning method.

3) Obtain the low dimensional embedding $\BY$ for the whole dataset inductively by Equation  \eqref{EQ:SubInduction}.

4) Threshold $\BY$ at zero.
}
\label{Alg1}
\end{algorithm}

\textbf{Connection to the Nystr\"om method}
As Equation \eqref{EQ:Induction}, the Nystr\"om eigenfunction by
Bengio et al.\ \cite{Nystrom2004} also generalizes to a new point by a linear combination of a set of low dimensional embeddings:
\begin{equation*}
\phi(\Bx) = \sqrt{n}\sum_{j = 1}^{n} \tilde{\mathrm{K}}(\Bx, \Bx_j)\mathbf{V}_{r}^{j}\mathbf{\Sigma}_r^{-1}.
\end{equation*}
For LE,  $\mathbf{V}_r$ and $\mathbf{\Sigma}_r$ correspond to the top
$r$ eigenvectors and eigenvalues of a normalized kernel matrix $\tilde{K}$ with
\[
  \tilde{K}_{ij} = \tilde{ {k}}(\Bx_i, \Bx_j) = \frac{1}{n}\frac{\w(\Bx_i,\Bx_j)}{\sqrt{E_{\Bx}[\w(\Bx_i, \Bx)]E_{\Bx}[\w(\Bx, \Bx_j)]}}.
\]
In AGH \cite{AGH2011}, the formulated hash function was proved to be the corresponding Nystr\"om eigenfunction with the approximate low-rank affinity matrix.
The  Laplacian eigenmaps latent variable model (LELVM) \cite{LELVM2007} also formulated the out-of-sample mappings for LE in a manner
similar to \eqref{EQ:Induction} by combining
latent variable models.
    Both of these methods, and the proposed one, can thus be seen as applications
    of the Nystr\"om method. {\em Note, however, that the suggested method differs in
    that it is not restricted to spectral methods such as LE, and that
    the present study aims to learn binary hash functions for similarity-based search
    rather than dimensionality reduction.}
    LELVM \cite{LELVM2007} cannot be applied to other embedding
    methods other than LE.

\subsection{Stochastic neighborhood preserving hashing}
\label{SEC:tSNE}
In order to demonstrate the proposed approach,  a hashing method based on t-SNE \cite{tSNE2008}, a non-spectral embedding method, is derived below.
T-SNE is a modification of stochastic neighborhood embedding (SNE) \cite{SNE2002} which aims to overcome the tendency of that method to crowd points together in one location.
 t-SNE provides an effective technique  for visualizing data and dimensionality reduction, which is capable of preserving local structures in the high dimensional data while retaining some global structures \cite{tSNE2008}. These properties make t-SNE a good choice for nearest neighbor search. Moreover, as stated in \cite{venna2010information}, the cost function of t-SNE in fact maximizes the \textit{smoothed recall} \cite{venna2010information} of query points and their neighbors.

The original t-SNE does not scale well, as it has a time complexity which is quadratic in $n$.  More significantly, however, it has a non-parametric form, which means that there is no simple function which may be applied to out-of-sample data in order to calculate their coordinates in the embedded space.
As was proposed in the previous subsection, one
first applies t-SNE \cite{tSNE2008} to the base set $\BB$,
\begin{equation}
\min_{\BY_{\BB}}  \sum_{\Bx_i \in \BB} \sum_{\Bx_j \in \BB} p_{ij} \log\bigg(\frac{p_{ij}}{q_{ij}}\bigg).
\label{EQ:tSNE}
\end{equation}
Here $p_{ij}$ is the symmetrized conditional probability in the high dimensional space, and
 $
q_{ij}
$ is the joint probability defined using the t-distribution in the low dimensional embedding space.
The optimization problem \eqref{EQ:tSNE} is easily solved by a gradient descent procedure\footnote{See details in \cite{tSNE2008}. A Matlab implementation of t-SNE is provided by the authors of \cite{tSNE2008} at \url{http://homepage.tudelft.nl/19j49/t-SNE.html.}}.
 After obtaining the embeddings $\BY_{\BB}$ of samples $\Bx_i \in \BB$, the hash codes for the entire dataset can be easily computed %
using~\eqref{EQ:hash_fun}.
This method  is labelled \IDH-tSNE.

\subsection{Hashing with relaxed similarity preservation}
\label{SEC:app_cons}
As in the last subsection, one can compute $\BY_{\BB}$ considering local smoothness only within $\BB$. Based on equation \eqref{EQ:Induction}, in this subsection, $\BY_{\BB}$ is alternatively computed  by considering the smoothness both within $\BB$ and between $\BB$ and $\BR$. As in \cite{Olivier2005},
the objective can be easily obtained by modifying \eqref{EQ:SH} as:
\begin{align}
\label{EQ:appro_cons}
\mathcal{C}(\BY_{\BB}) & = \sum_{\Bx_i,\Bx_j \in \BB} \w(\Bx_i, \Bx_j) \|\By_i -\By_j\|^2 & (\mathcal{C}_{\BB\BB}) \notag \\ \notag
 & + \lambda \sum_{\Bx_i \in \BB, \Bx_j \in \BR} \w(\Bx_i, \Bx_j) \|\By_i -\By_j\|^2  & (\mathcal{C}_{\BB\BR})\\
\end{align}
where $\lambda$ is the trade-off parameter.
$\mathcal{C}_{\BB\BB}$ enforces smoothness of the learned embeddings within  $\BB$ while $\mathcal{C}_{\BB\BR}$ ensures the smoothness between $\BB$ and $\BR$. This formulation is actually a relaxation of \eqref{EQ:SH}, by discarding the  part which minimizes the dissimilarity within $\BR$ (denoted as $\mathcal{C}_{\BR\BR}$).
$\mathcal{C}_{\BR\BR}$ is ignored since computing the similarity matrix within $\BR$ costs $\mathit{O}(n^2)$ time. The smoothness between points in $\BR$ is implicitly ensured by \eqref{EQ:SubInduction}.

Applying equation  \eqref{EQ:SubInduction} for $\By_j, j \in \BR$ to \eqref{EQ:appro_cons}, one obtains the following problem
\begin{align}
\label{EQ:Objective1}
\min &\, \trace(\BY_{\BB}^\T (\BD_{\BB} - \BW_{\BB}) \BY_{\BB})\\ \notag
 + &\lambda \, \trace(\BY_{\BB}^\T(\BD_{\BB\BR} - \bar{\BW}_{\BR\BB}^\T \BW_{\BR\BB})\BY_{\BB}),
\end{align}
where $\BD_{\BB} = \diag(\BW_{\BB}\Bone)$ and $\BD_{\BB\BR} = \diag(\BW_{\BB\BR}\Bone)$ are both $m \times m$ diagonal matrices.
Taking the constraint in \eqref{EQ:SH}, one obtains
\begin{align}
\label{EQ:Obj_main}
\min_{\BY_{\BB}} \, & \trace(\BY_{\BB}^\T(\BM + \lambda \BT)\BY_{\BB})\\ \notag
\st & \BY_{\BB}^\T \BY_{\BB} = m\BI
\end{align}
where $\BM = \BD_{\BB} - \BW_{\BB}$, $\BT = \BD_{\BB\BR} - \bar{\BW}_{\BR\BB}^\T \BW_{\BR\BB}$.
The optimal solution $\BY_{\BB}$ of the above problem is easily obtained by identifying the $r$ eigenvectors of $\BM + \lambda\BT$ corresponding to the smallest eigenvalues (excluding the eigenvalue 0 with respect to the trivial eigenvector $\Bone$)\footnote{The parameter $\lambda$ is set to 2 in all experiments.}.
This method is named \IDH-LE in the following text.

\begin{figure}
\centering
\includegraphics[width = 0.48\textwidth]{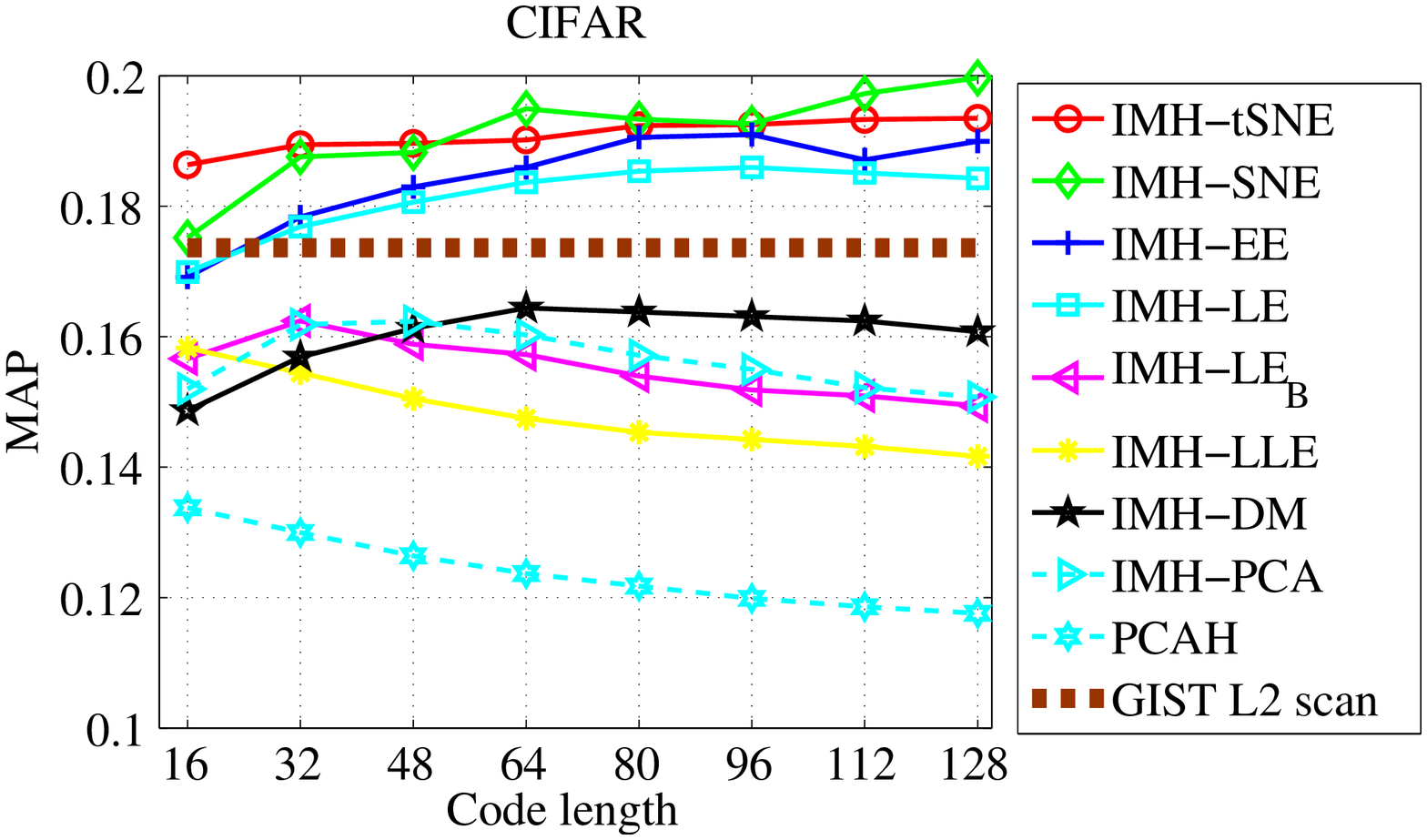}
\caption{Comparison among different manifold learning methods within
the proposed \IDH hashing framework on  CIFAR-10. \IDH with the linear PCA
(\IDH-PCA) and PCAH \cite{SSH2012} are also evaluated for comparison.
For clarity, for \IDH-LE in Section \ref{SEC:app_cons},  \IDH with the original LE algorithm on the base set $\BB$ is termed as \IDH-LE$_\BB$.
IMH-DM is IMH with the diffusion maps of \cite{lafon2006diffusion}. The base set size is set to 400.
}
\label{Fig:EmbeddingSelection}
\end{figure}

\subsection{Manifold learning methods for hashing}
\label{Embedding_selection}
In this section,  different manifold learning methods are compared for hashing within the proposed \IDH framework.
Figure~\ref{Fig:EmbeddingSelection} reports the comparative results in mean of average precision (MAP).
For comparison,  the linear PCA is also evaluated within the framework (\IDH-PCA in the figure).
  As can be clearly seen that, \IDH-tSNE, \IDH-SNE and \IDH-EE   perform slightly better than
  \IDH-LE (Section \ref{SEC:app_cons}). This is mainly because
these three methods are able to preserve local neighborhood structure while, to some extent, preventing data points from crowding together.
It is promising that all of these methods perform  better than an exhaustive $\ell_2$ scan using the uncompressed GIST features.

Figure~\ref{Fig:EmbeddingSelection} shows that
 LE (\IDH-LE$_\BB$ in the figure), the most widely used embedding method in hashing, does not perform as well as a variety of other methods (\eg, t-SNE), and in fact performs worse than PCA, which is a linear technique.
This is not surprising because LE (and similarly LLE) tends to collapse large portions of the data (and not only nearby samples in the original space) close together in the low-dimensional space.
The results are consistent with the analysis in \cite{tSNE2008,EE2010}.
Based on the above observations, we argue that  manifold learning methods (\eg t-SNE, EE), which not only preserve local similarity  but also force dissimilar data apart in the low-dimensional space, are more effective than the popular LE for hashing.
Preserving the global structures of the data is critical for the manifold learning algorithms like t-SNE used in the proposed IMH framework.
The previous work called spline regression hashing (SRH) \cite{liu2012spline} also exploited both the local and global data similarity structures of data via a Laplacian matrix, which can decrease over-fitting, as discussed in \cite{liu2012spline}. SRH captures the global similarity structure of data through constructing global non-linear hash functions, while, in the proposed method, capturing the global structures (\ie, by t-SNE) is only applied on the small base set and the hash function is formulated by a locally linear regression model.

It is interesting to see that \IDH-PCA outperforms PCAH \cite{SSH2012}
by a large margin, despite the fact that PCA is performed on the
whole training data set by PCAH. This shows that the generalization capability of \IDH
based on a very small set of data points also  works for  linear
dimensionality methods.

\section{Experimental results}
\label{Sec:exp}
 \IDH is evaluated on four large scale image datasets: CIFAR-10\footnote{\url{http://www.cs.toronto.edu/~kriz/cifar.html}}, MNIST
, SIFT1M \cite{SSH2012} and GIST1M\footnote{\url{http://corpus-texmex.irisa.fr/}}.
The MNIST dataset consists of $70,000$ images, each of 784 dimensions, of handwritten digits from `0' to `9'.   As a subset of the well-known 80M tiny image collection \cite{80Mtiny2008}, CIFAR-10 consists of 60,000 images which are manually labelled as 10 classes with $6,000$ samples for each class. Each image in this dataset is represented by a GIST feature vector \cite{GIST2001} of dimension $512$. For MNIST and CIFAR-10, the whole dataset is split into a test set with $1,000$ samples and a training set with all remaining samples.

Nine hashing algorithms is compared, including the proposed \IDH-tSNE, \IDH-LE and seven other unsupervised state-of-the-art methods: PCAH \cite{SSH2012}, SH \cite{SH08}, AGH \cite{AGH2011} and STH \cite{STH2010}, BRE \cite{BRE2009}, ITQ \cite{PCA-ITQ}, Spherical Hashing (SpH) \cite{spherical2012}.
The provided codes and suggested parameters according to the authors of these methods are used.
Due to the high computational cost of BRE and high memory cost of STH,  $1,000$ and $5,000$ training points are sampled for these two methods respectively.
The performance is measured by MAP or
precision and recall curves for \textit{Hamming ranking} using 16 to
128 hash bits. The results  for \textit{hash lookup}
using a Hamming radius within 2 by F1 score \cite{F108}: $F_1 = 2
( precision \cdot recall) / (precision + recall)$ are also reported. Ground truths are
defined by the category information for the labeled datasets MNIST
and CIFAR-10, and by Euclidean neighbors for SIFT1M and GIST1M.

\begin{figure}
 \centering
 \includegraphics[width = 0.235\textwidth]{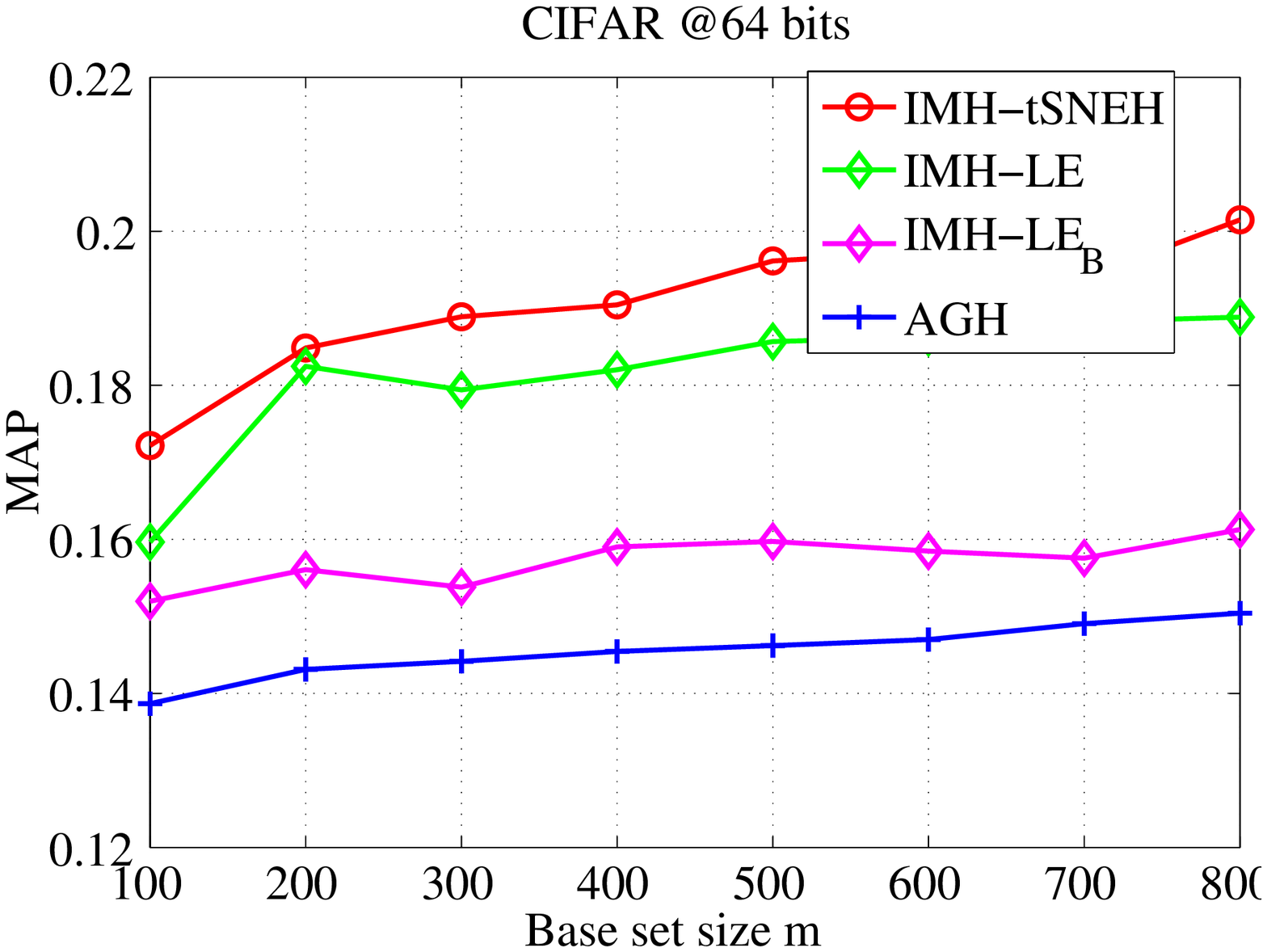}
 \includegraphics[width = 0.235\textwidth]{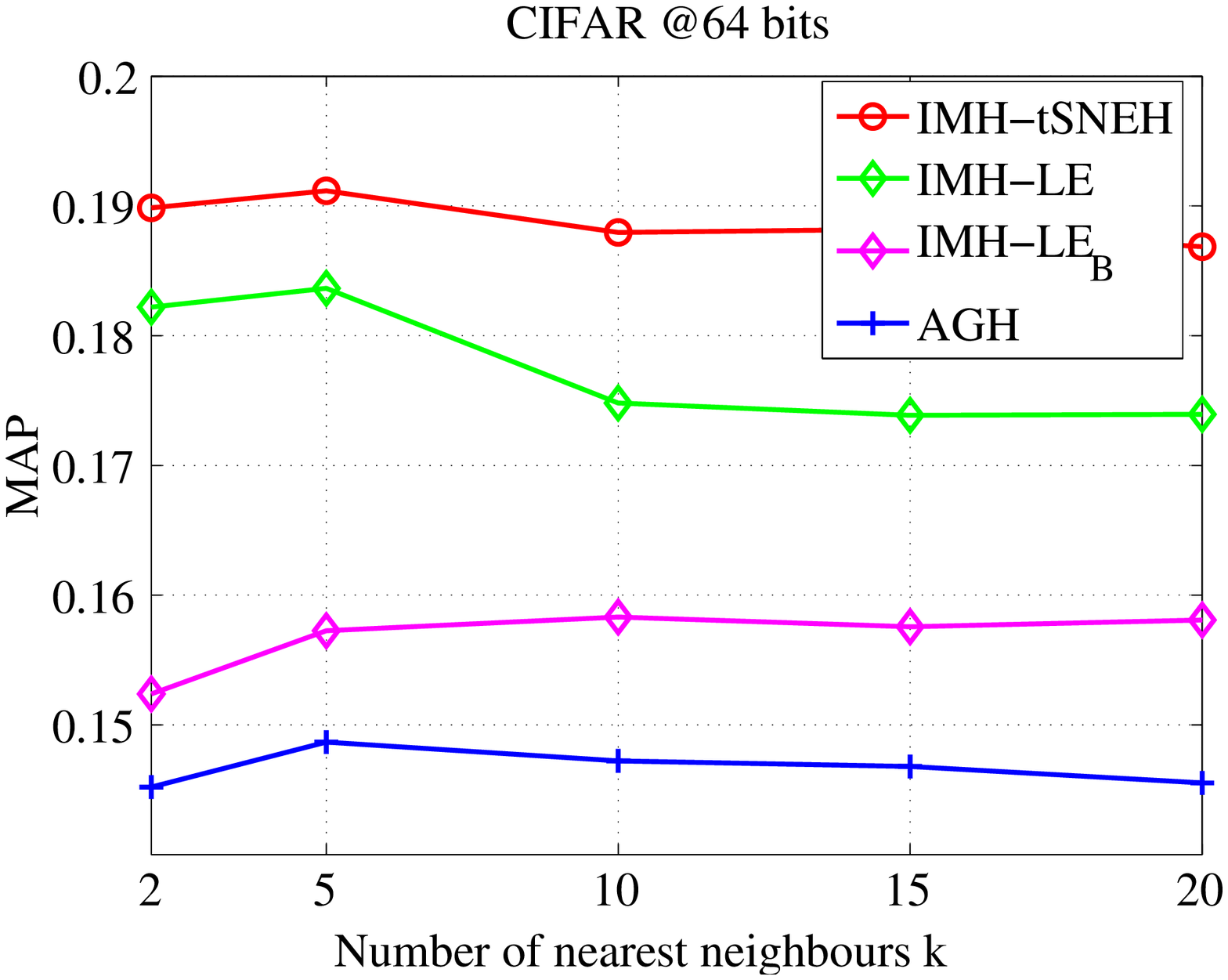}
\label{compare_anchor_num}
\caption{MAP results versus varying  base set size $m$ (left, fixing $k = 5$) and  number of nearest base points $k$ (right, fixing $m = 400$) for the proposed methods and AGH.
The comparison is conducted on the CIFAR-10 dataset using 64-bits .}
\label{Fig:compare_anchornum}
\end{figure}

\begin{table}[]
\caption{MAP (\%)  of different base generating methods: random sampling vs. K-means.
The comparison is performed on the CIFAR-10 dataset  with code lengths from 32 to 96 and base set size 400. Average results (standard deviation) are given based on 10 runs.}
\centering
\resizebox{0.5\textwidth}{!}{
\begin{tabular}{cc|cccc}
\hline\hline
Bits & Method & \IDH-LE$_\BB$ & \IDH-LE & \IDH-tSNE & AGH\\
 \hline
\multirow{3}{*}{32} &
 Random & 14.07 (0.27)  & 16.20 (0.19) & 17.26 (0.43) & -\\
 & K-medians &15.14 (0.15) & 17.08 (0.21) & 17.32 (0.27) &-\\
 &K-means & 16.05 (0.16)  & 17.48 (0.44) & {\bf 18.38 (0.36)} & 15.76 (0.17)\\
 \hline
 \multirow{3}{*}{64} &
 Random & 14.64 (0.17) & 16.98 (0.21)& 16.93 (0.49)& - \\
 & K-medians &15.43 (0.12) &17.11 (0.17) &17.62 (0.31) &-\\
 &K-means & 15.90 (0.11) & 18.20 (0.37)  & {\bf 19.04 (0.27)} & 14.55 (0.13)\\
 \hline
 \multirow{3}{*}{96} &
 Random & 14.76 (0.20)  &  17.02 (0.23)& 17.21(0.48)& -\\
   & K-medians &15.42 (0.19) &17.21 (0.28) &17.73 (0.28) &-\\
 &K-means & 15.46 (0.09) &  18.56 (0.34) & {\bf 19.41 (0.23)} & 13.98 (0.10)\\
 \hline\hline
\end{tabular}
}

\label{compare_random_kmeans}
\end{table}

\subsection{Base selection}
In this section,   the CIFAR-10 dataset is taken as an example to compare different base generation methods and different base sizes for the proposed methods. AGH is also evaluated here for comparison.
Table~\ref{compare_random_kmeans} compares three methods for generating base point sets: random sampling, K-medians and K-means on the training data.  One can easily see that the performance of the proposed methods using K-means is better at all code lengths than that using other two methods. It is also clear from Table~\ref{compare_random_kmeans} that the K-medians algorithm achieves better performance than random sampling, however still inferior to K-means.

Different from K-means that constructs each base point by averaging the data points in the corresponding cluster, random sampling and K-medians generate the base set with real data samples in the training set. However this property does not help K-medians obtain better hash codes than  K-means. This is possibly because K-means produces the cluster centers with minimum quantization distortions at each group of the data points. Another advantage of K-means is it is much more efficient than K-medians.

As  can also be seen, even with base set by random sampling, the proposed methods outperform AGH in all cases but one. Due to the superior results and high efficiency in practice, the base set is generated by K-means in the following experiments.

From Figure~\ref{Fig:compare_anchornum},
it is clear that the performance of the proposed methods is consistently improved with the increasing base set size $m$, which is consistent with the analysis of the prototype algorithm. One can observe that,
the performance does not change significantly with  the number of nearest base points $k$.
 It is also clear that \IDH-LE$_\BB$, which only enforces smoothness in the base set, does not perform as well as \IDH-LE, which also enforces smoothness between the base set and training set.

To further investigate the impact of the base set size $m$, in both performance and efficiency,  the proposed methods are performed with large variance of $m$ on  CIFAR-10 and MNIST.
The results for the proposed IMH-tSNE are shown in Table~\ref{Tab:EXT}. As can be clearly seen, for the dataset of CIFAR-10, the MAP score by IMH-tSNE improves consistently with the increasing base set size $m$ when $m \leq 400$, while does not change dramatically with larger $m$. On MNIST, the performance of IMH-tSNE is also consistently improved when $m$ increases. Same as on CIFAR-10, the MAP score does not significantly change with larger $m$.

In terms of computational efficiency, on both datasets, the training time cost is considerably increased with larger $m$. For example,   IMH-tSNE costs about 4.3 seconds with 400 base samples, while costs more than 21 seconds with 1,000 base samples on the CIFAR-10 dataset. For the retrieval task, the testing time is more crucial. As shown in Table~\ref{Tab:EXT}, the testing time is in general linearly increased with $m$. With a small $m$, the testing is very efficient. When with a large $m$ (\eg, $m \geq 3000$), IMH-tSNE needs more than one milliseconds to compute the binary codes of the query, which is not scalable for large-scale tasks. Take both performance and computational efficiency into account, for the remainder of this paper, the settings $m = 400$ and $k = 5$ are used for the proposed methods, unless otherwise specified.

 Figure~\ref{Fig:manifolds} shows the t-SNE embeddings of the base set (400 points generated by K-means) and the embeddings of the whole MNIST dataset computed by the proposed inductive method (without binarization). As can be seen from the right figure that most of the points  are close to their corresponding clusters  (with respect to 10 digits). This observation shows that, in the low-dimensional embedding space, the proposed method can well preserve the local manifold structure of the whole dataset based on a relatively small base set.

\begin{table*}
\centering
\caption{Impact of the base set size $m$ of IMH-tSNE on the retrieval performance and computational efficiency. The base set is generated by K-means. The results with 64 bits are reported, based on 10 independent runs. The experiments are conducted on a desktop PC with a 4-core 3.40GHZ CPU and 32G RAM.}
\resizebox{\textwidth}{!}{
\begin{tabular}{cl|lllllllllll}
\hline\hline
&$m$ & 100 &200 & 400  & 800 & 1000 & 1500 & 2000 &3000 & 4000 & 5000&10000\\
\hline
\multirow{3}{*}{CIFAR}&MAP (\%) & 17.55&  18.19&    19.52   &    19.99&    20.11 &   20.13 &   20.05 & 20.23&    20.37   &20.45&20.85\\
&Training Time (s) & 0.7754 & 1.3913  &  4.2850   & 13.5576 &  21.3487  & 53.5668 &  97.6969  &229.6100&  419.4027 & 664.0867&2701.8001\\
&Testing Time (ms)& 0.0057 & 0.0095  &  0.0157&      0.0254 &   0.0308   & 0.0475  &  0.0621 & 0.1045&    0.1388 &   0.1981&0.3634\\
\hline
\multirow{3}{*}{MNIST}& MAP (\%) &57.01&  67.72   & 77.78  &   79.71 &   80.40  &  82.52 &   82.10&  86.47  &  85.66   & 87.37 &87.74\\
&Training Time (s) & 0.9546 & 1.6804   & 4.6302 &   814.7084 &  22.7297  & 54.7796 & 102.1419 & 232.8753  &423.6395&  669.6443 &2707.8012\\
&Testing Time (ms) & 0.0074 & 0.0122  &  0.0182    &  0.0330   & 0.0394   & 0.0579   & 0.0984  &0.1222 &   0.2032 &   0.2293&0.4357\\
\hline\hline
\end{tabular}
}
\label{Tab:EXT}
\end{table*}

\begin{figure}
\includegraphics[width=0.24\textwidth]{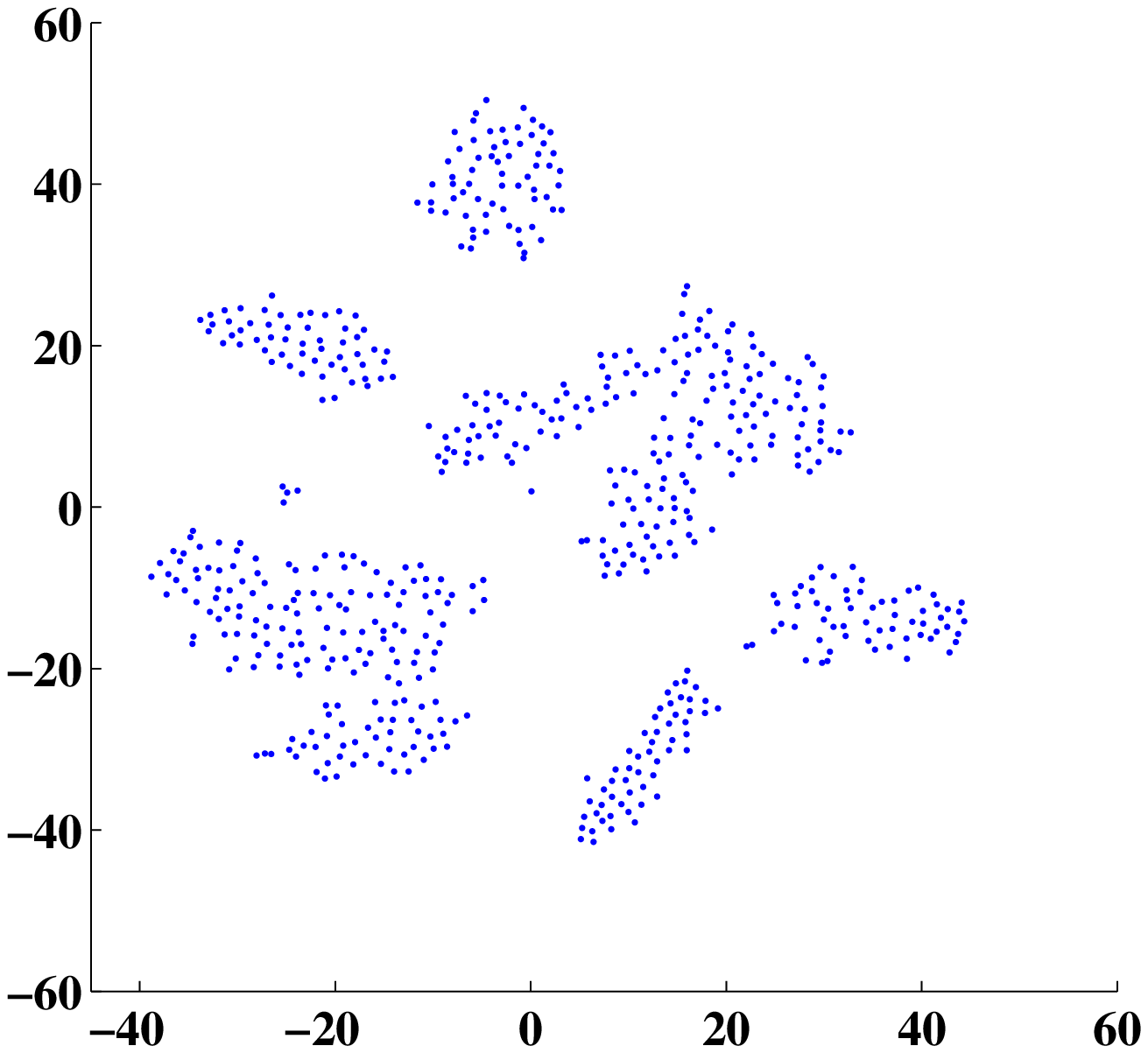}
\includegraphics[width=0.24\textwidth]{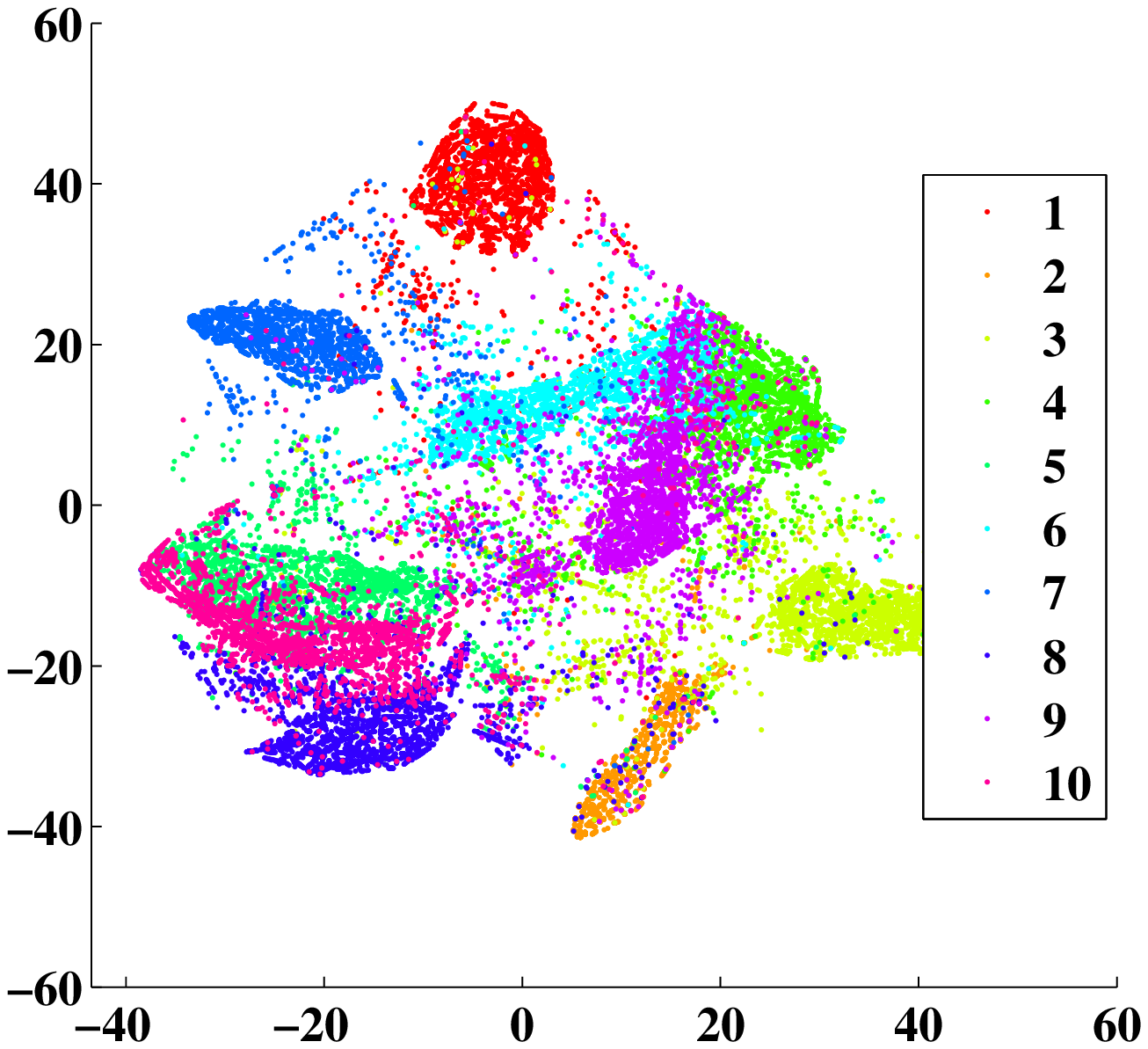}
\caption{Visualization of the digits in MNIST in 2D by t-SNE: (left) t-SNE embeddings of the base set (400 data points generated by K-means); (right) embeddings of all the 60,000 samples computed by the proposed inductive method. In the right, the labels of the digits are indicated with 10 different colors.}
\label{Fig:manifolds}
\end{figure}

\subsection{Results on CIFAR-10 dataset}
The comparative results  based on  MAP   for  Hamming ranking  with code lengths from 16 to 128
bits
are reported in Figure~\ref{cifar}.
It can be seen that the proposed \IDH-LE and \IDH-tSNE perform best in all cases.
Among the proposed algorithms, the LE based \IDH-LE is inferior to the t-SNE based \IDH-tSNE.
 \IDH-LE is still much better than AGH and STH, however. ITQ performs better than SpH and BRE on this dataset, but is still inferior to \IDH.
  SH and PCAH perform worst in this case. This is because SH
relies upon its uniform data
assumption while PCAH simply generates the hash hyperplanes by PCA directions, which does not explicitly capture the similarity information. The results are consistent with the complete precision and recall curves shown in the supplementary material.
The $F_1$ results for hash lookup with Hamming radius 2 are also reported.
It is can be seen that \IDH-LE and \IDH-tSNE also outperform all other methods by large margins.  BRE and AGH obtain  better results than the remaining methods, although
the performance of all methods drop as code length grows.

Figure~\ref{cifar_pr} shows the  precision and recall curves of Hamming ranking for the compared methods. STH and AGH obtain relatively high precisions when a small number of samples are returned, however precision drops significantly as the number of retrieved samples increases. In contrast,
 \IDH-tSNE, \IDH-LE and ITQ achieve  higher precisions with relatively larger numbers of retrieved points.

Figure~\ref{Fig:returned_images} shows the qualitative results of \IDH and related methods on some sample queries.
As can be seen, \IDH-tSNE achieves the best search quality in term of visual relevance.

\begin{figure}[t]
\centering
\includegraphics[width = 0.235\textwidth]{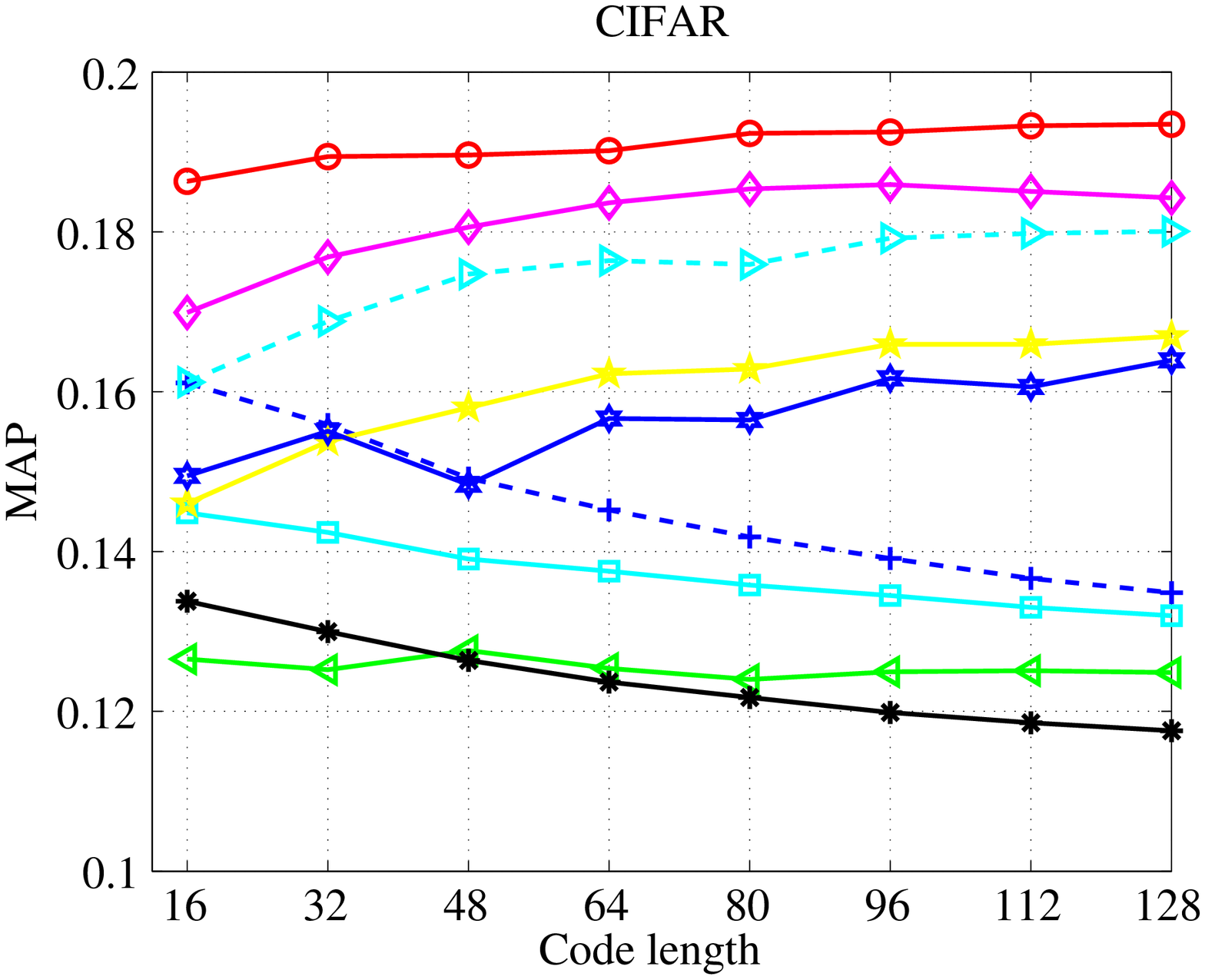}
\includegraphics[width = 0.235\textwidth]{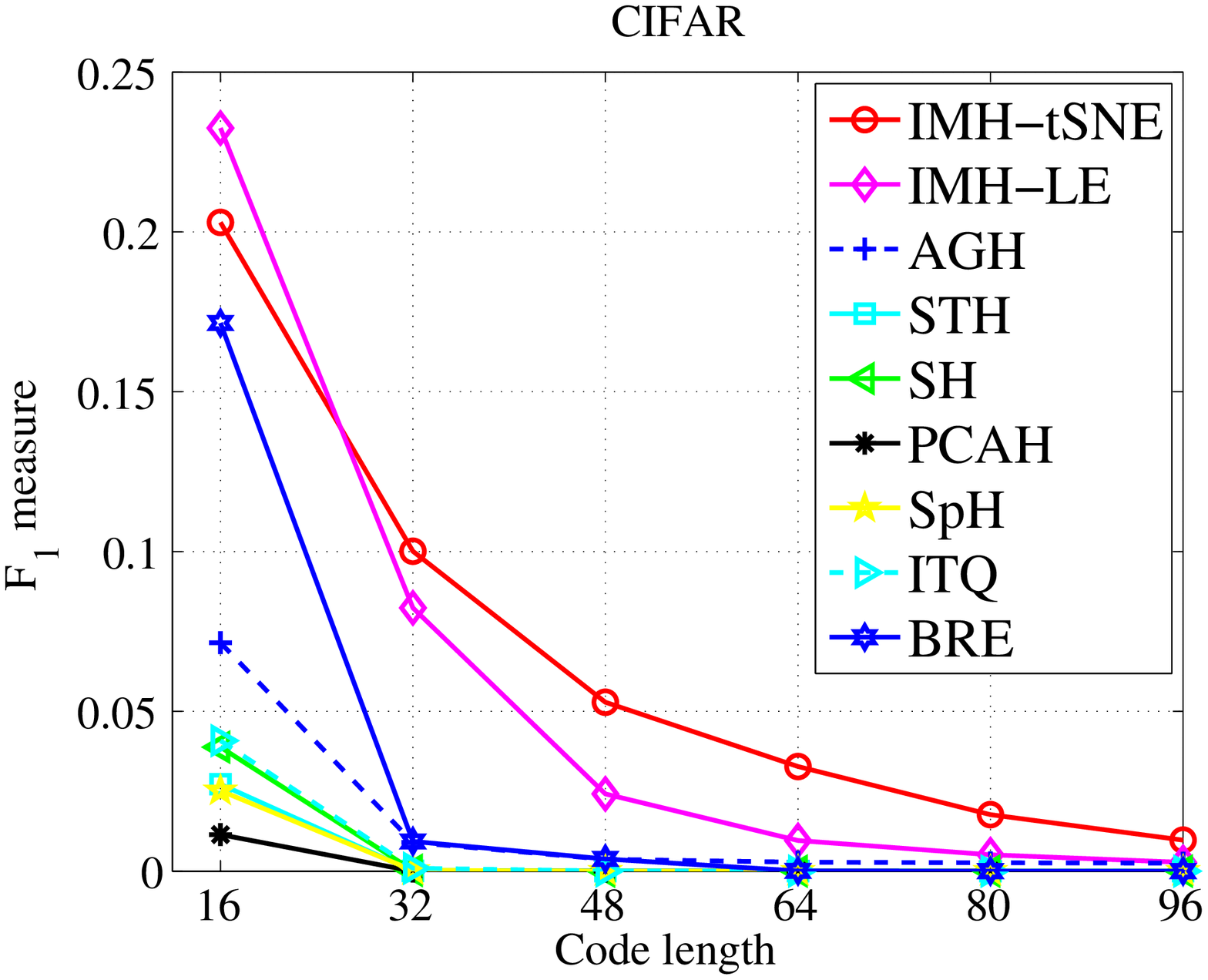}
\caption{Comparison of different methods on CIFAR-10 based on MAP (left) and $F_1$  (right) for varying code lengths.
}
\label{cifar}
\end{figure}

\begin{figure}
\centering
\includegraphics[width = 0.235\textwidth]{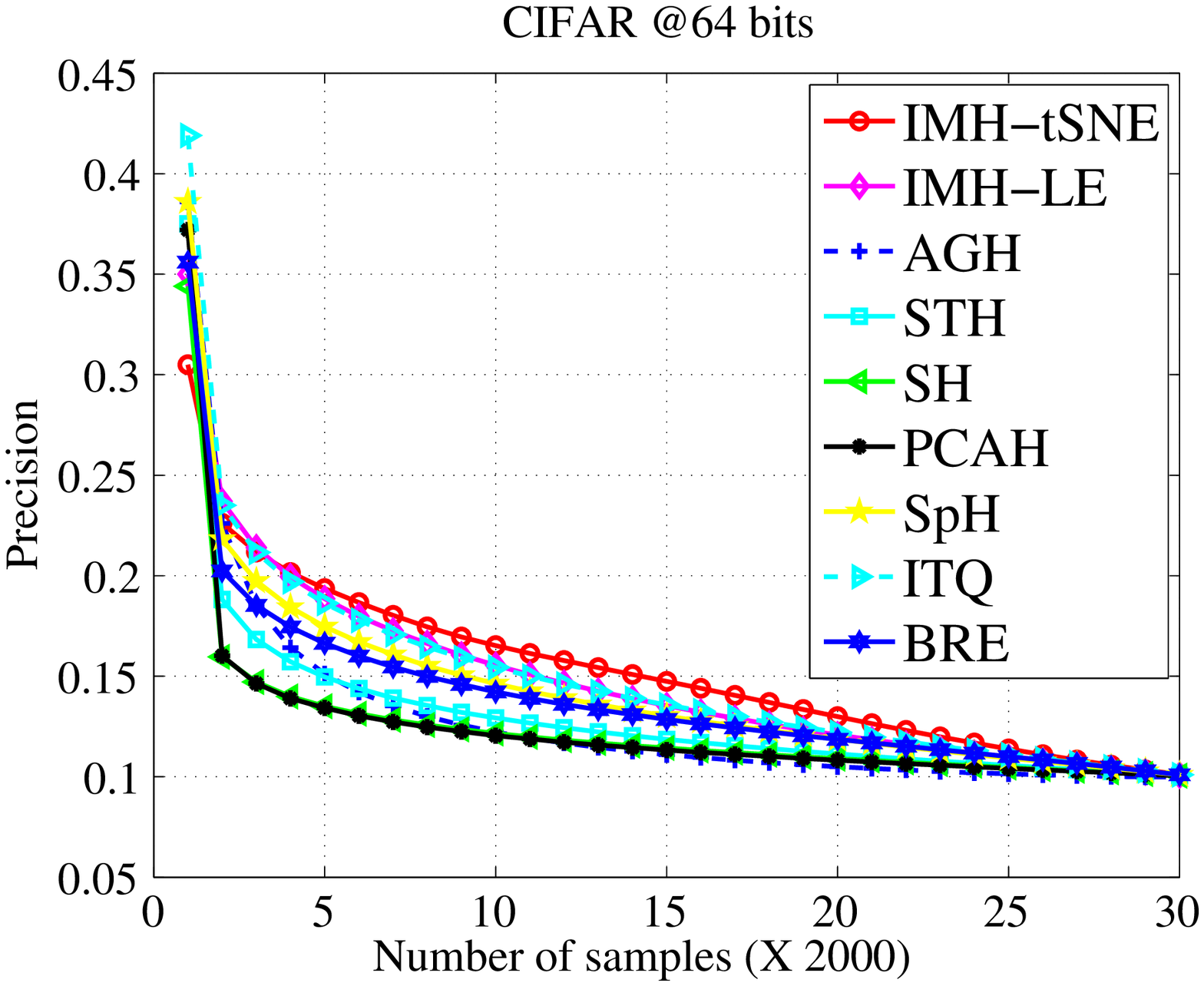}
\includegraphics[width = 0.235\textwidth]{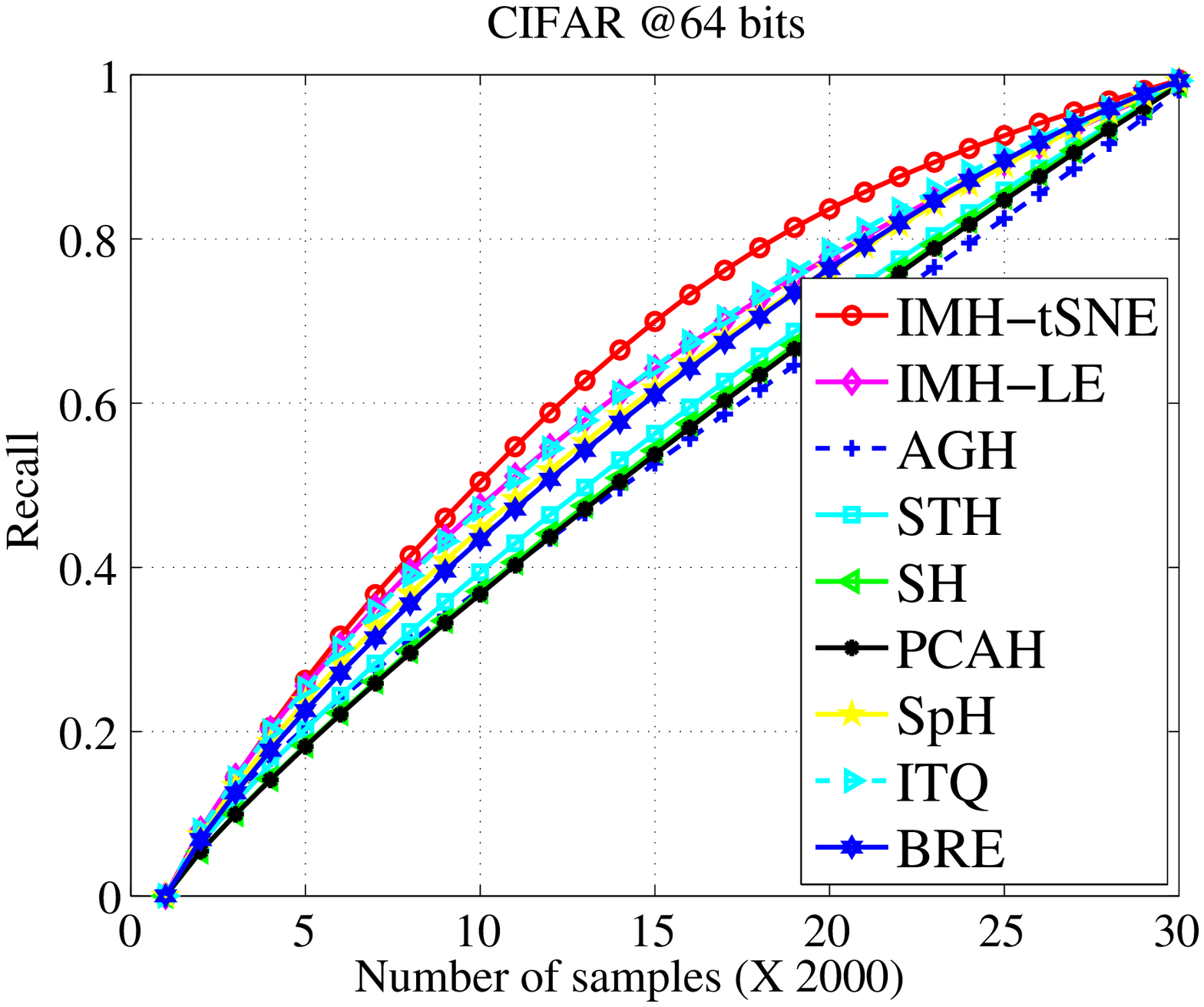}
\caption{Comparison of different methods on CIFAR-10 based on precision (left) and recall (right) using 64-bits.
Please refer to the complementary for complete results for other code lengths.}
\label{cifar_pr}
\end{figure}

\begin{figure*}
\centering
\includegraphics[width = \textwidth]{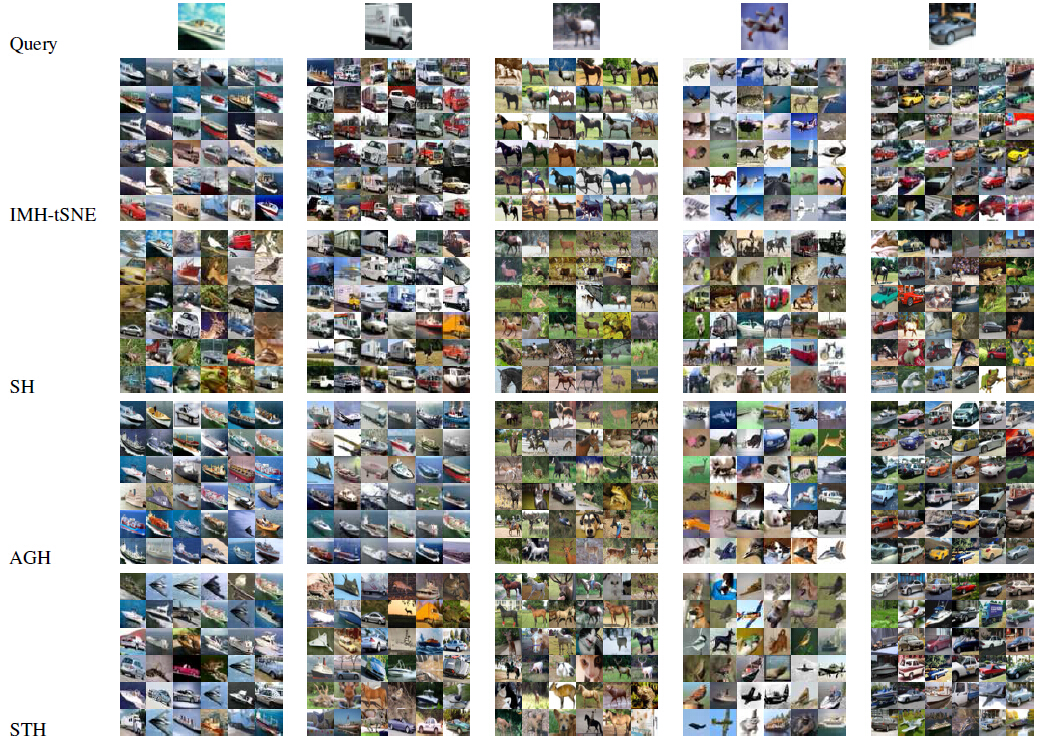}
\caption{The query images (top row) and the retrieved images by various methods with 32 hash bits. }
\label{Fig:returned_images}
\end{figure*}

\begin{figure}[t]
\centering
\includegraphics[width = 0.235\textwidth]{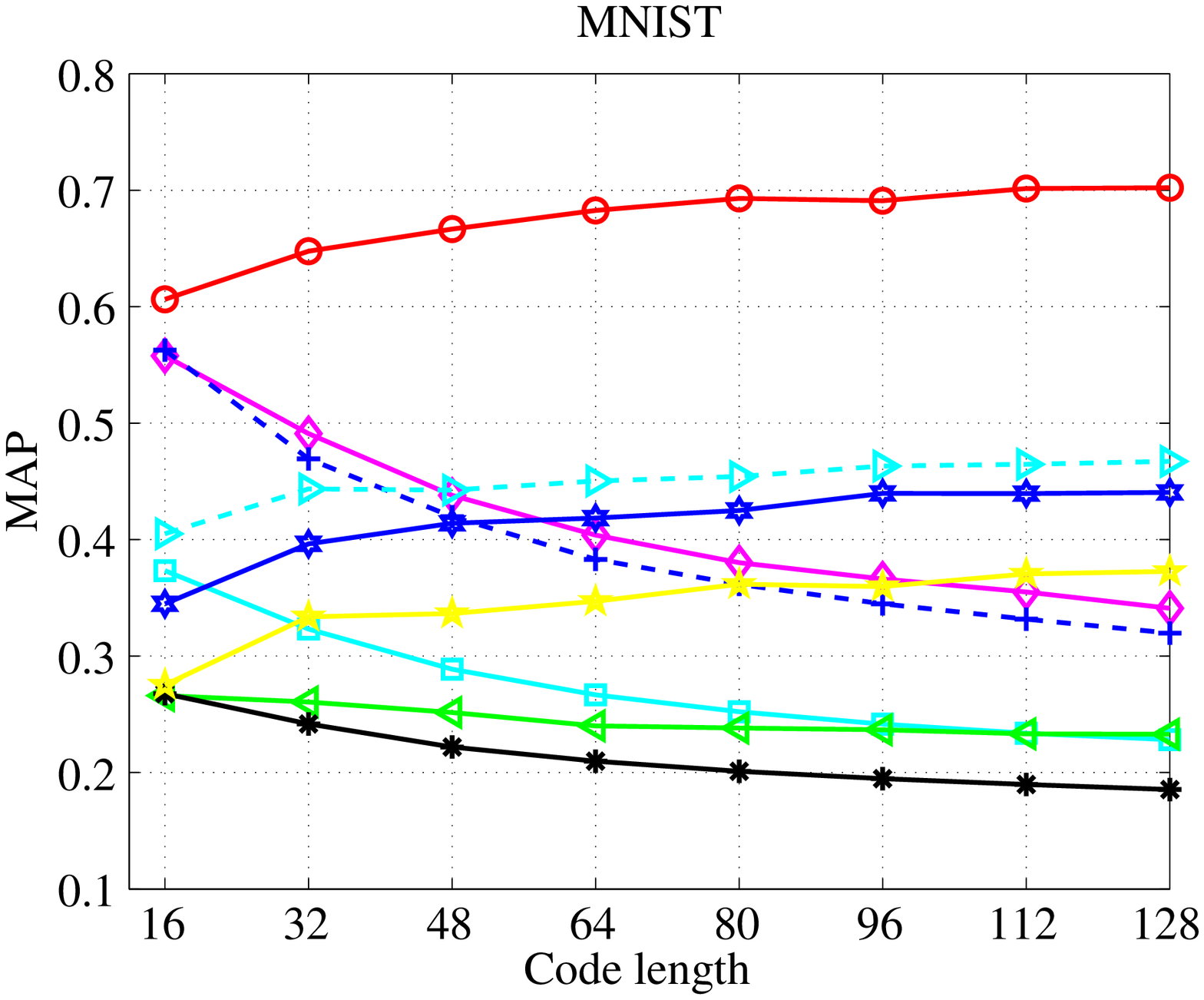}
\includegraphics[width = 0.235\textwidth]{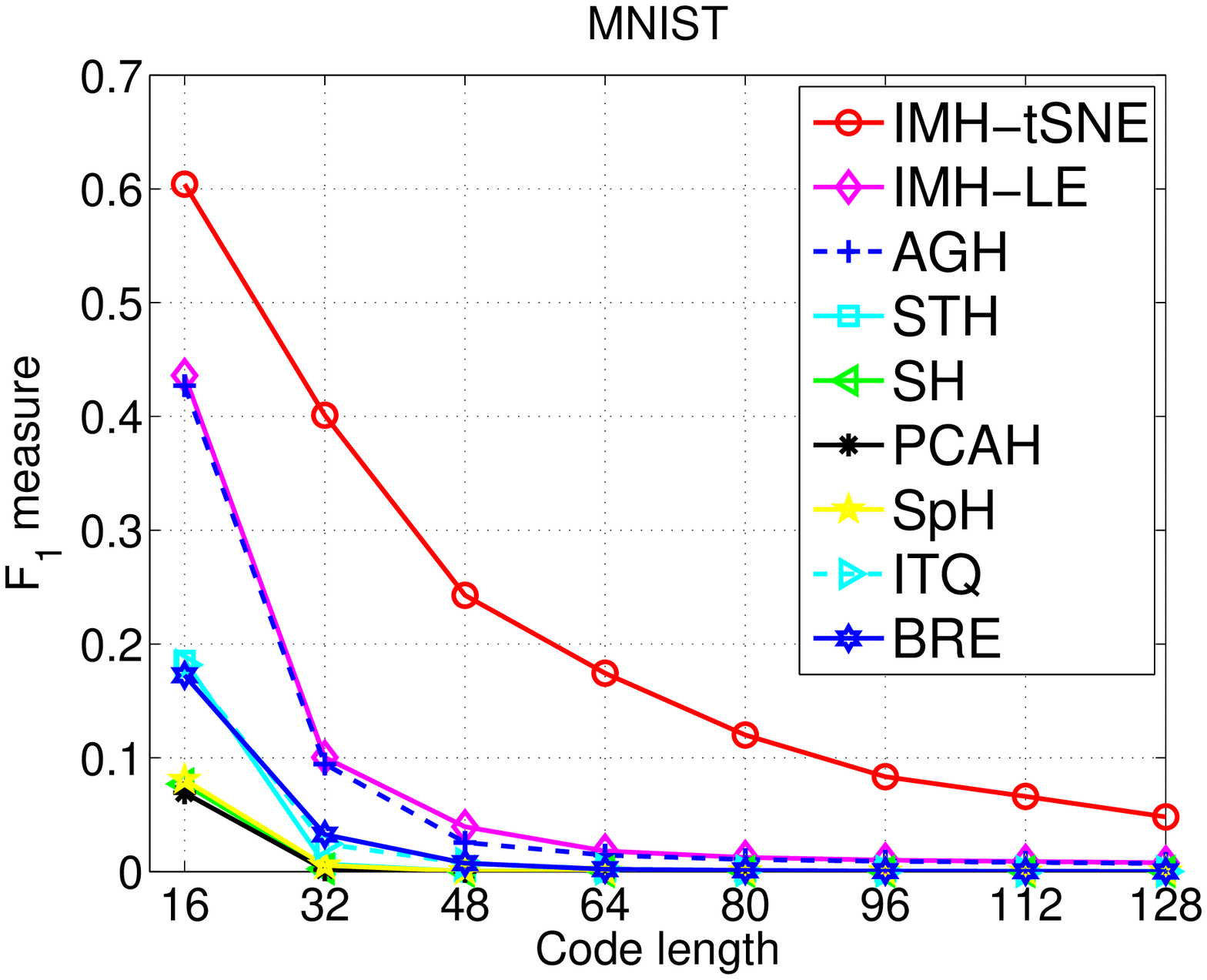}
\caption{Comparison of different methods on the MNIST dataset using MAP (left) and $F_1$  (right) for varying code lengths.
}
\label{mnist}
\end{figure}

\begin{table}[]
\caption{Comparison of training and testing times (in seconds) on MNIST with 70K 784D feature points. K-means dominates the cost of  AGH and \IDH (8.9 seconds), which can be conducted in advance in practice. The experiments are based on a desktop PC with a 4-core 3.07GHZ CPU and 8G RAM.}
\centering
\begin{tabular}{ccc|cc}
\hline\hline
Method & \multicolumn{2}{c|}{Train time} & \multicolumn{2}{c}{ Test time}\\
&64-bits  & 128-bits  & 64-bits  & 128-bits \\
\hline
\IDH-LE&9.9 & 9.9& $5.1 \times 10^{-5}$& $3.8 \times 10^{-5}$ \\
\IDH-tSNE&16.7&20.2 &$2.8 \times 10^{-5}$ & $3.1 \times 10^{-5}$\\
SH&6.8 & 16.2&$5.8 \times 10^{-5}$ &$1.8 \times 10^{-4}$  \\
STH&266.1& 485.4& $1.8\times 10^{-3}$& $3.6 \times 10^{-3}$ \\
AGH&9.5 & 9.5&$4.7 \times 10^{-5}$ &$5.5 \times 10^{-5}$ \\
PCAH&3.8 &4.1 &$5.7 \times 10^{-6}$ &$1.2 \times 10^{-5}$\\
SpH & 19.7&41.0 &$1.3 \times 10^{-5}$ & $2.0 \times 10^{-5}$\\
ITQ &10.4 &20.3 &$6.9 \times 10^{-6}$ & $1.1 \times 10^{-5}$\\
BRE &418.9 &1731.9 &$1.2\times 10^{-5}$ &$2.4 \times 10^{-5}$ \\
\hline\hline
\end{tabular}
\label{Tab:time_mnist}
\end{table}
\subsection{Results on MNIST dataset}
The MAP and $F_1$ scores for these compared methods are reported in Figure~\ref{mnist}. As in Figure~\ref{cifar}, \IDH-tSNE achieves the best results. It is clear  that, on this dataset \IDH-tSNE outperforms \IDH-LE by a large margin, which increases as code length increases.
This further demonstrates the advantage of t-SNE as a tool for hashing by embedding high dimensional data into a low dimensional space. The dimensionality reduction procedure  not only preserves the local neighborhood structure, but also reveals important global structure (such as clusters) \cite{tSNE2008}. Among the four LE-based methods, while \IDH-LE shows a small advantage over AGH, both methods achieve much better results than STH and SH. ITQ and BRE obtain high MAPs with longer bit
lengths,
but they still perform less well for the hash look up $F_1$.
 PCAH  performs worst  in terms of both MAP and the $F_1$ measure.  Refer to the supplementary material for the complete precision and recall curves which validate the observations here.

\textit{\textbf{Efficiency}} %
Table~\ref{Tab:time_mnist} shows training and testing time on the MNIST dataset for various methods, and shows that
the linear method, PCAH, is fastest.
 \IDH-tSNE is slower than \IDH-LE, AGH and SH in terms of training time, however all of these methods have relatively low execution times and are much faster than STH and BRE. In terms of test time, both \IDH algorithms are comparable to other methods,  except STH which takes much more time to predict the binary codes by SVM on this non-sparse dataset.

\begin{figure}[]
\centering
\includegraphics[width = 0.235\textwidth]
{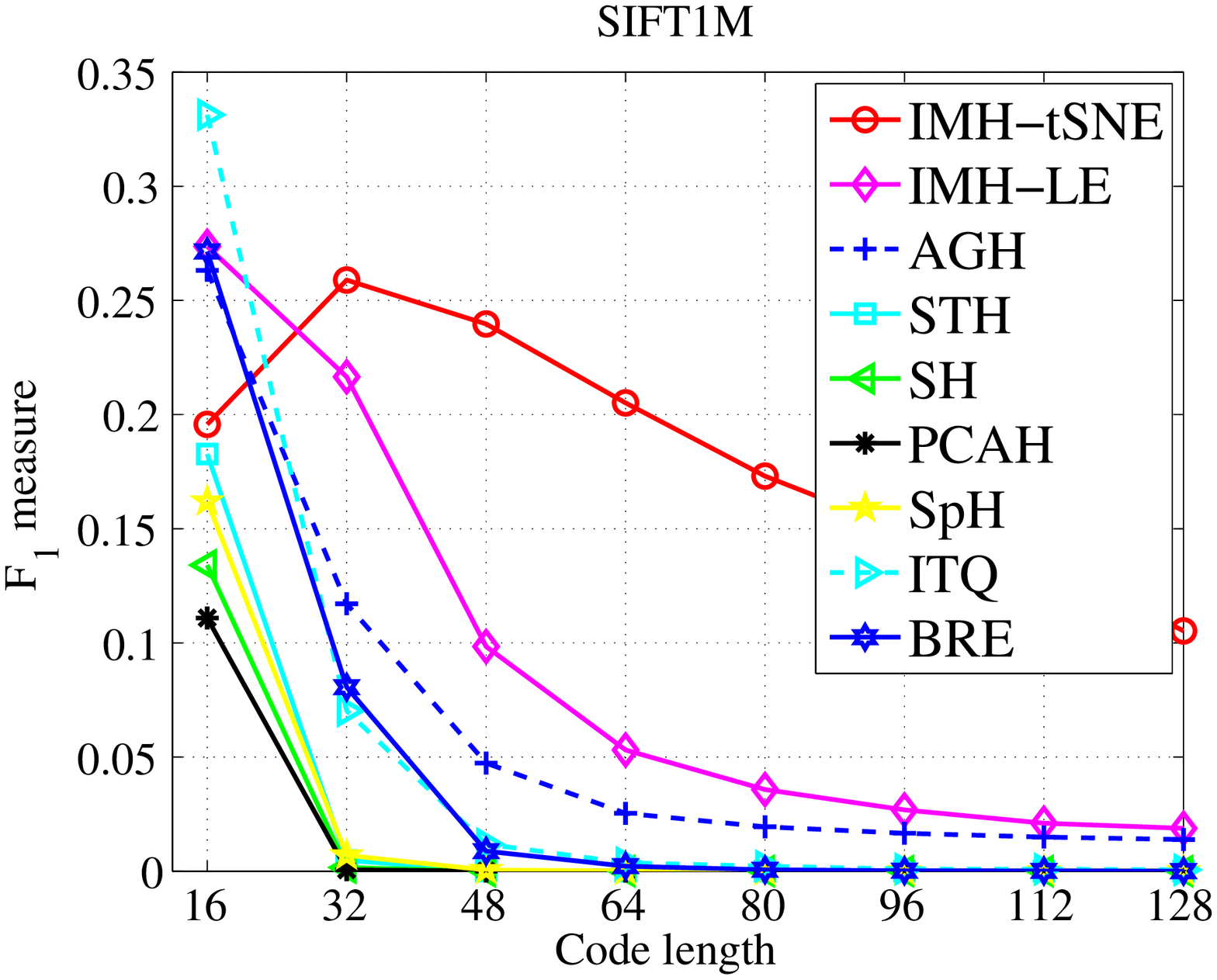}
\includegraphics[width = 0.235\textwidth]
{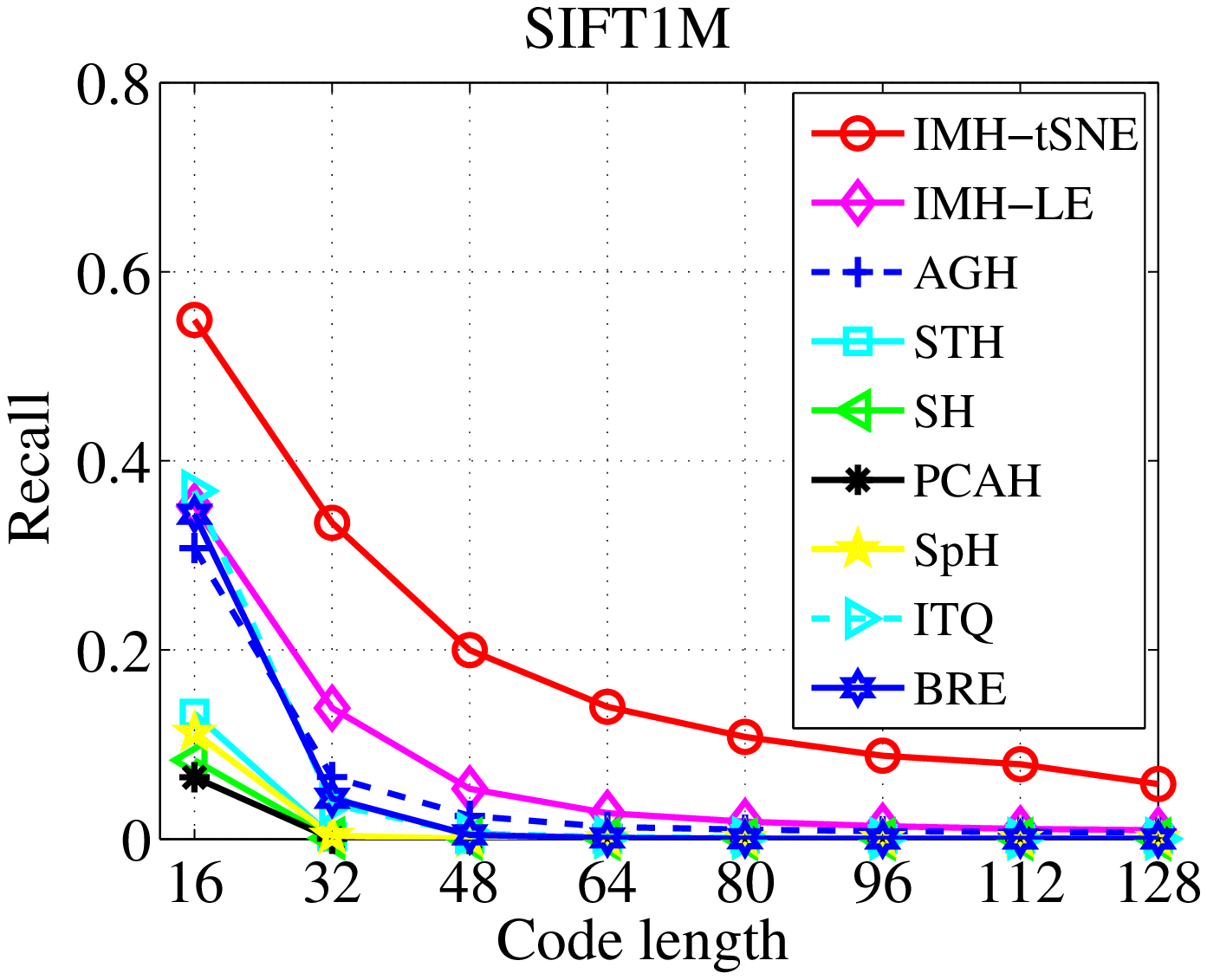}
\caption{Comparative results results on SIFT1M for $F_1$ (left) and recall (right) with Hamming radius 2. Ground truth is defined to be the closest 2 percent of points as measured by the Euclidean distance.}
\label{sift1M}
\end{figure}

\subsection{Results on SIFT1M and GIST1M}
SIFT1M  contains one million local SIFT descriptors extracted from a large set of images \cite{SSH2012}, each of which is represented by a 128D vector of histograms of gradient orientations.
GIST1M contains one million GIST features and
each feature is represented by a 960D vector.
 For both of these datasets, one million samples are used as training set and additional 10K are used for testing. As in \cite{SSH2012}, ground truth is defined as the
closest 2 percent of points as measured by the
Euclidean distance. For these two large datasets,  $1,000$ points are generated by K-means and $k$ is set as 2 for both \IDH and AGH.
The comparative results on SIFT1M and GIST1M are summarized in Figure~\ref{sift1M} and Figure~\ref{GIST1M}, respectively.
Again, \IDH consistently achieves superior
results in terms of both $F_1$ score and recall with Hamming radius 2. As can be seen that, the performance of most of these methods decreases dramatically with increasing code length as the Hamming spaces become more sparse, which makes the hash lookup fail more often. However \IDH-tSNE still achieves relatively high scores with large code lengths. If one looks at Figure~\ref{sift1M} (left), ITQ obtains the highest $F_1$ with 16-bits, however it decreases to near zero at 64-bits. In contrast, \IDH-tSNE still manages an $F_1$ of $0.2$.  Similar results are observed in the recall curves.

\begin{figure}[t]
\centering

\includegraphics[width = 0.235\textwidth]
{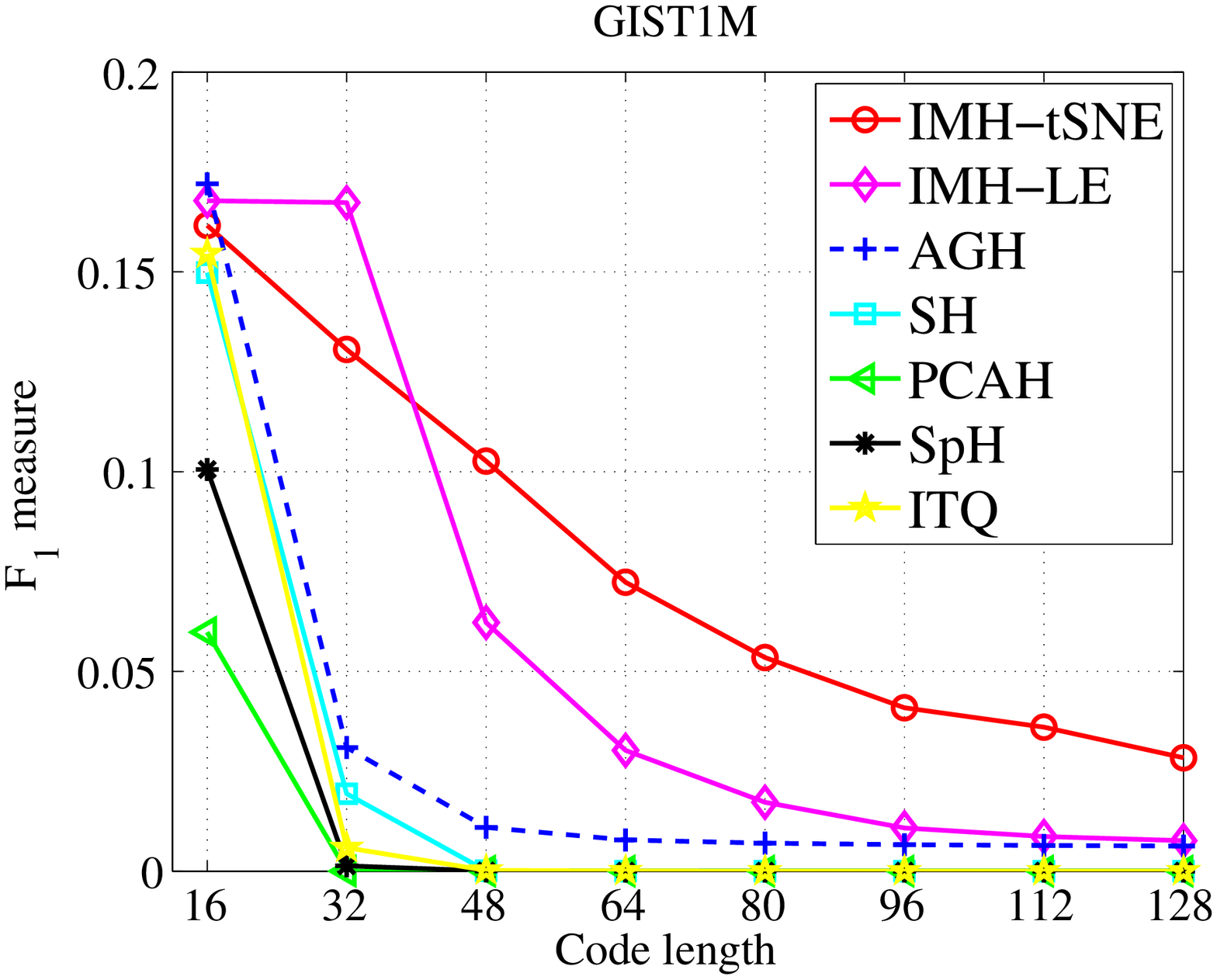}
\includegraphics[width = 0.235\textwidth]
{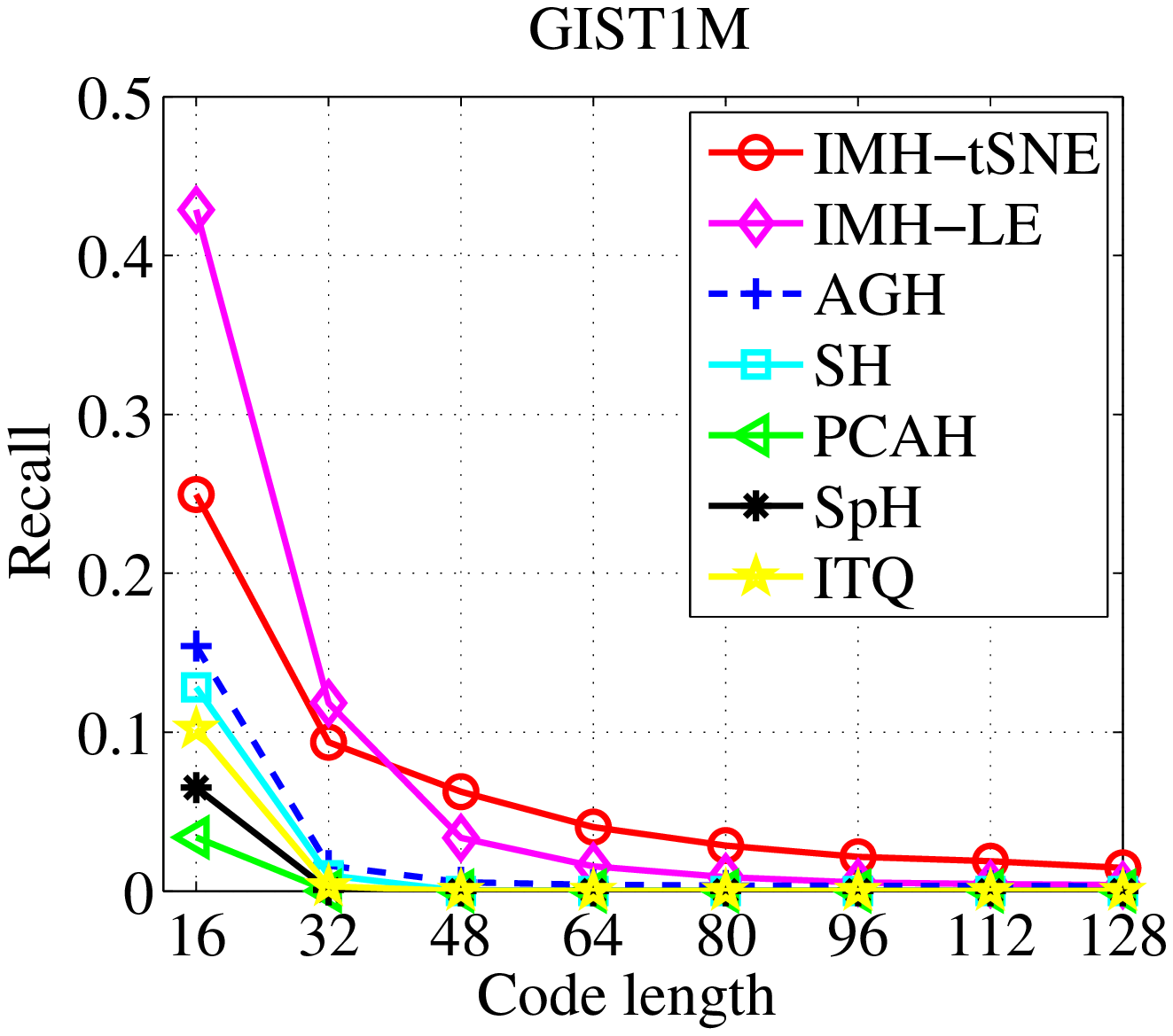}
\caption{Comparative results results on GIST1M by  $F_1$ (left) and recall (right) with Hamming radius 2. Ground truth is defined to be the closest 2 percent of points as measured by the Euclidean distance.}
\label{GIST1M}
\end{figure}

\begin{figure}
\centering
\includegraphics[width = 0.48\textwidth]{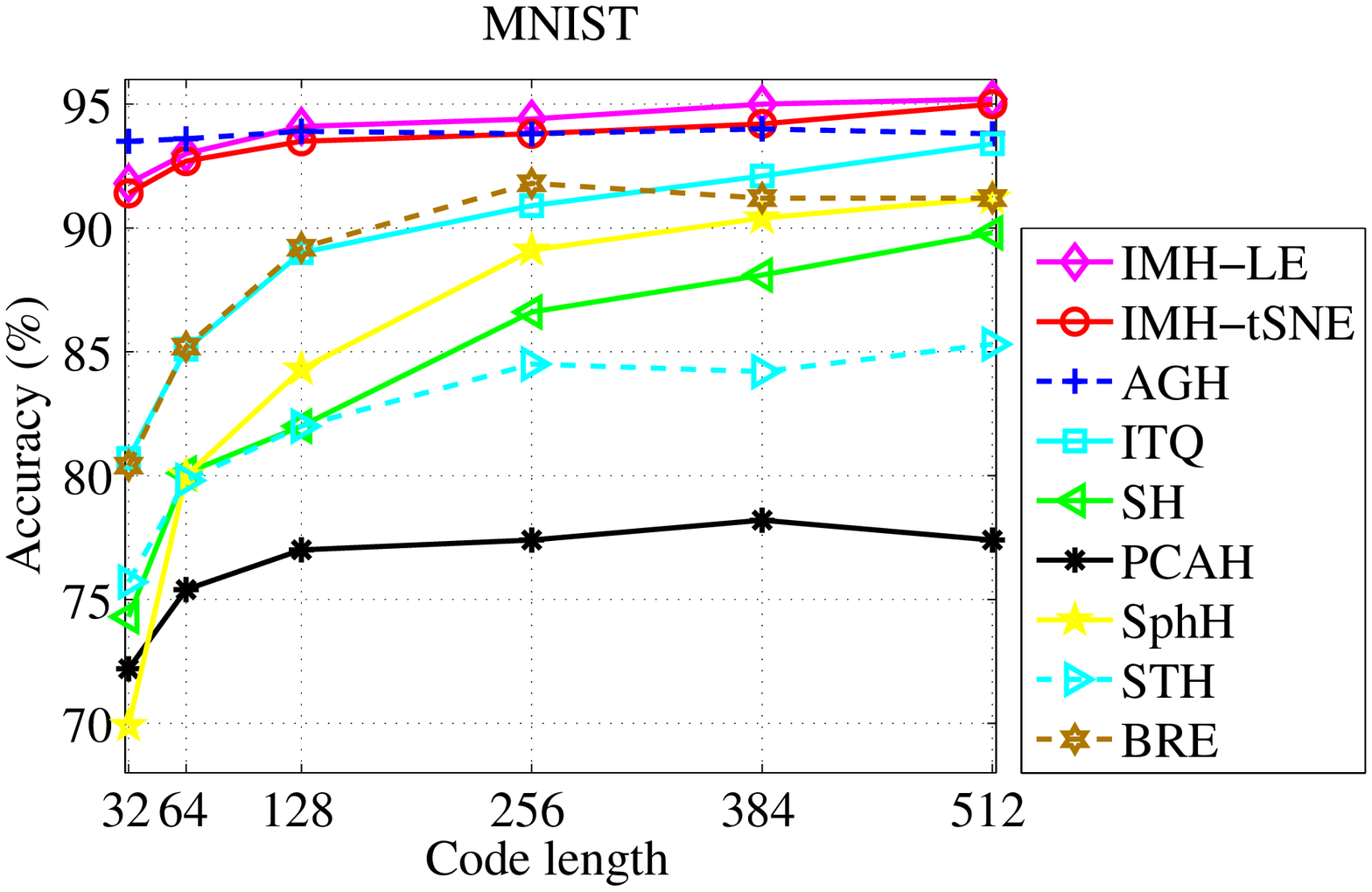}
\caption{Classification accuracy (\%) on MNIST with binary codes of various hashing methods by linear SVM.}
\label{calssification_mnist}
\end{figure}

\subsection{Classification on binary codes}
In order to demonstrate classification performance
a linear SVM is trained on the binary codes generate by \IDH for the MNIST data set.
In order to learn codes with higher bit
lengths
  for \IDH and AGH,  the base set size is set to $1,000$.
Accuracies of different binary
encodings
are shown in Figure~\ref{calssification_mnist}. Both \IDH and AGH achieve high accuracies on this dataset, although \IDH performs better with higher code
lengths.
In contrast, the best results of all other methods, obtained by ITQ,
are consistently worse than those for \IDH, especially for short code
lengths. Note that even with only 128-bit binary features \IDH obtains
a high 94.1\%. Interestingly,  the same classification rate of
94.1\% is obtained by applying the linear SVM to the uncompressed 784D features, which occupy several hundreds times as much space as the learned hash codes.

\begin{figure*}[]
\centering
\includegraphics[height = 5cm]{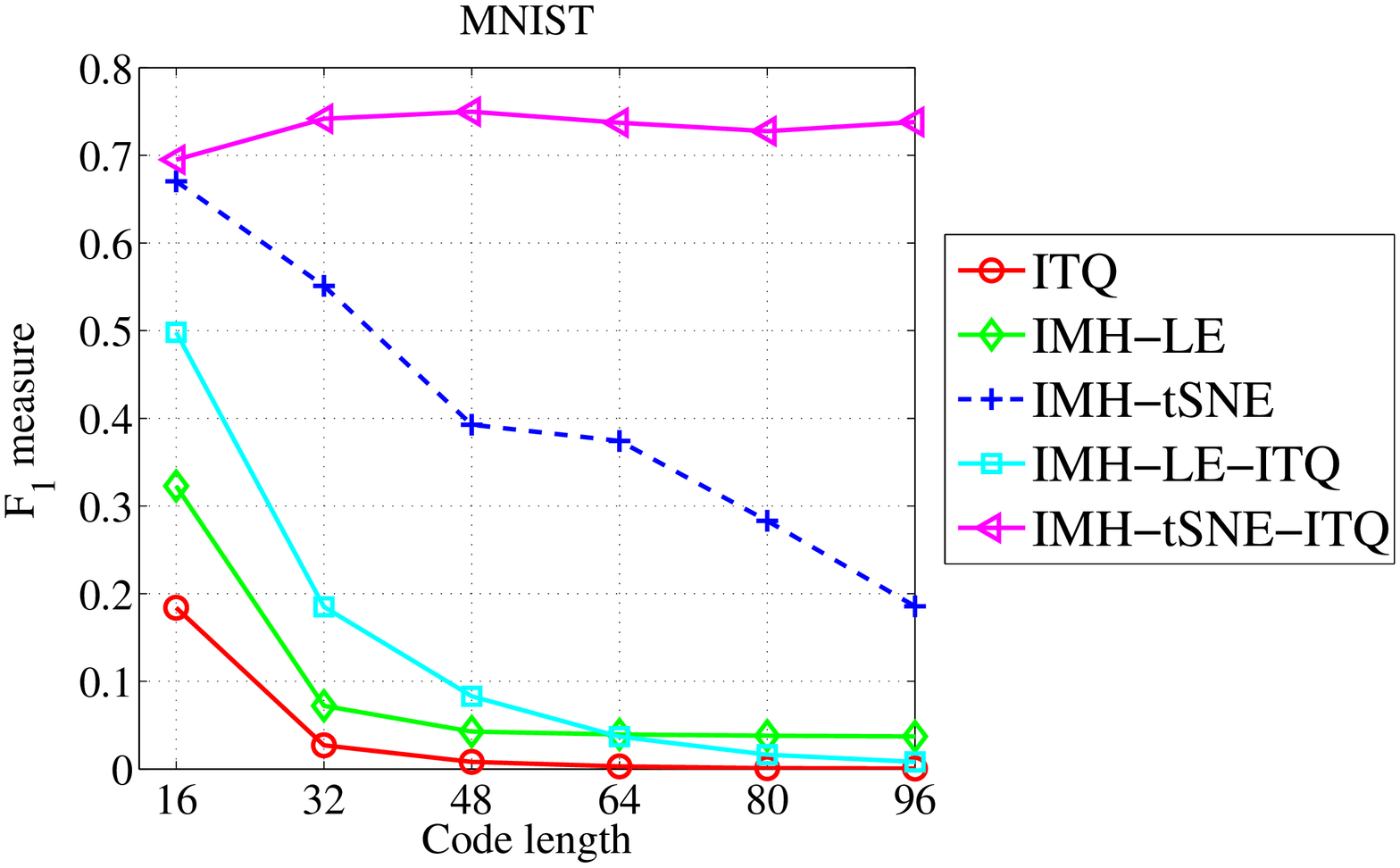}
\includegraphics[height = 5cm]{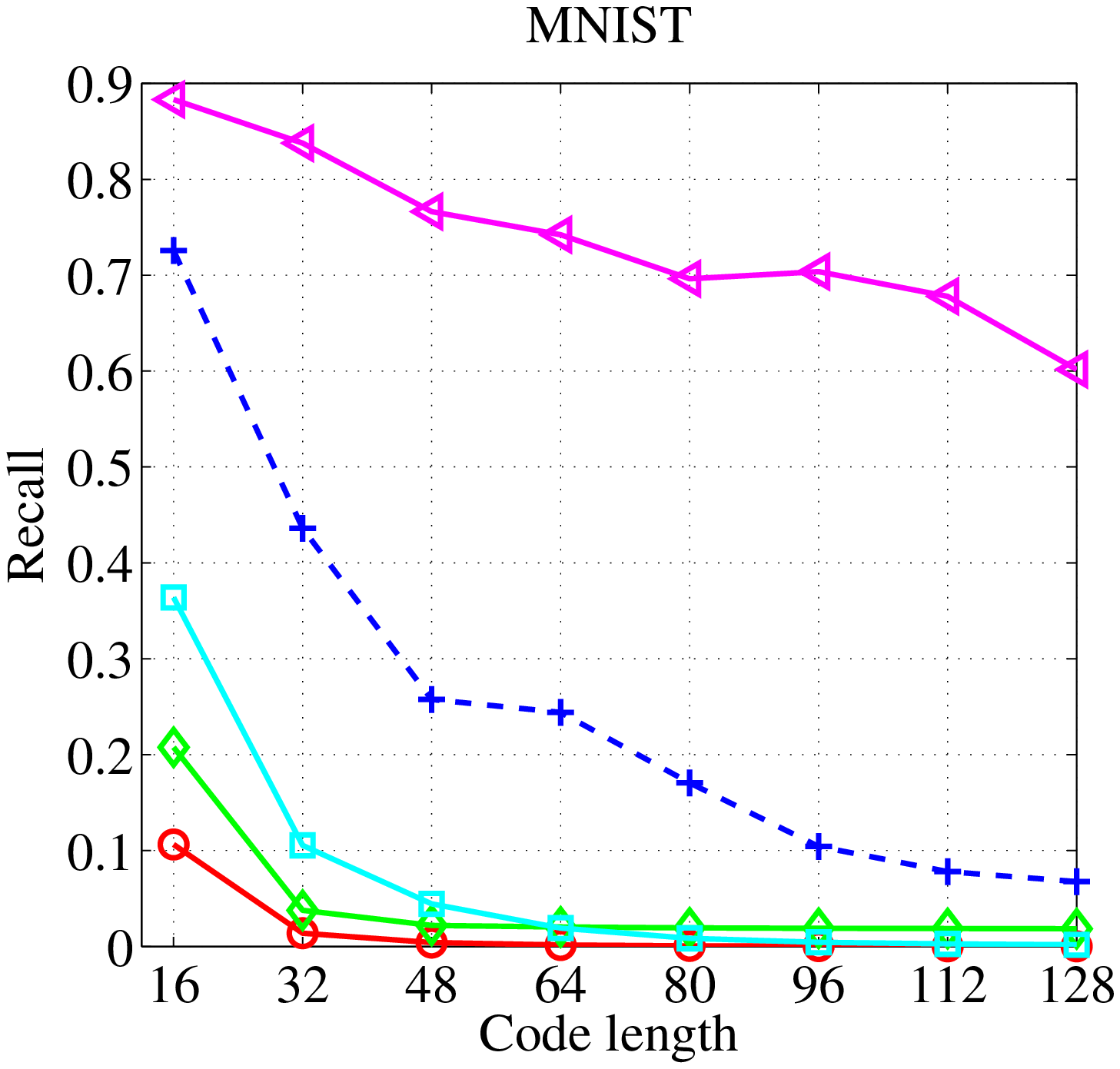}
\includegraphics[height = 5cm]{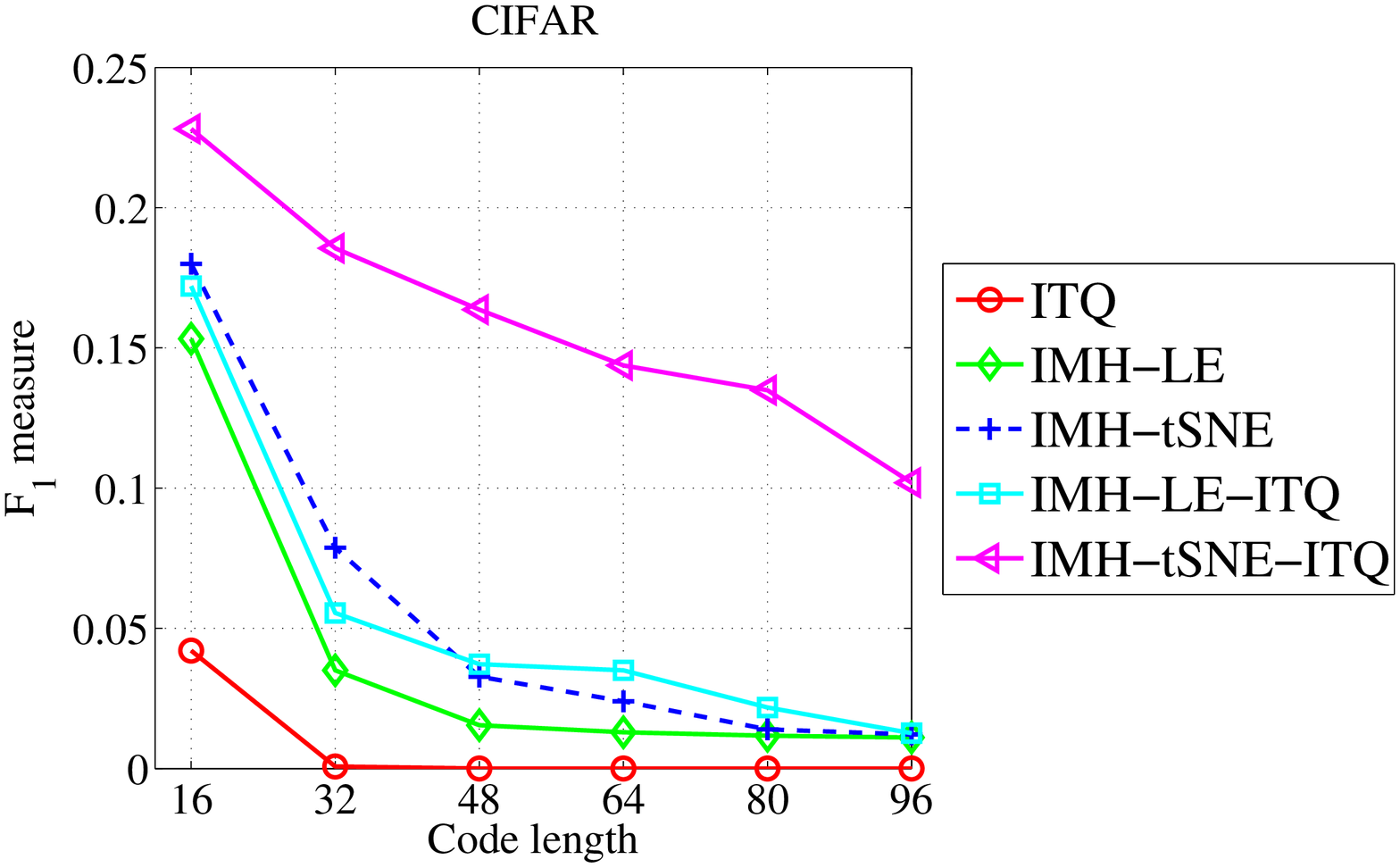}
\includegraphics[height = 5cm]{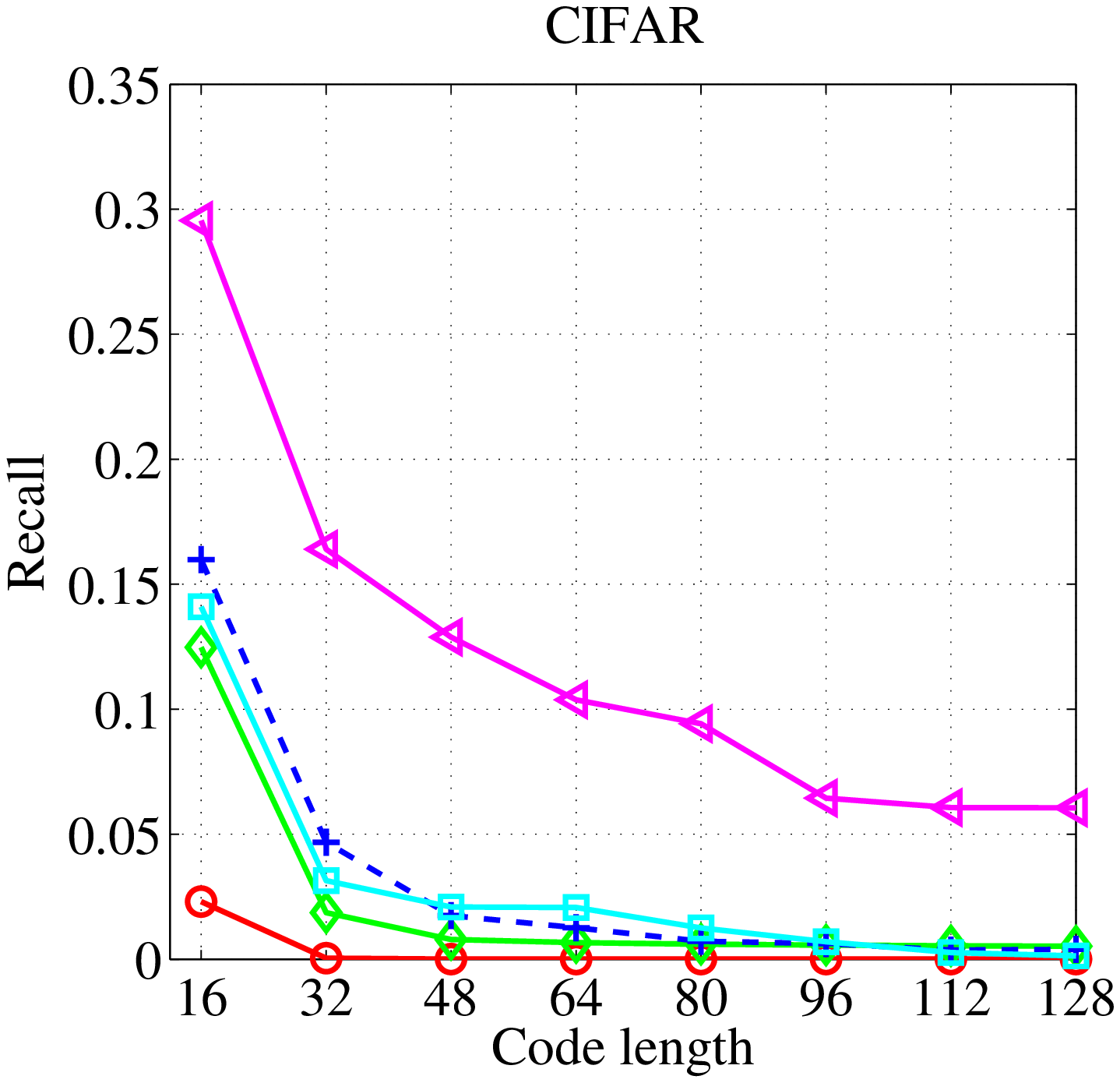}
\caption{Evaluation of IMH with ITQ rotations on the MNIST and CIFAR  dataset using $F_1$  and $recall$ with Hamming radius 2 for varying code lengths.
}
\label{mnist}
\end{figure*}
\section{Minimize the quantization distortion by learned rotations}
\label{Sec:ITQ}
In the above sections, the binary codes are obtained by directly thresholding the learned embeddings at zeros. The simple binarization may cause large quantization loss.
In this section, learned rotations are applied to minimize the quantization errors. That is, normalize the data points such that they are zero-centred in the embedded space and then orthogonally rotate the normalized embeddings for binarization. Orthogonal rotation has been adopted by many methods (\eg \cite{jegou2010aggregating,PCA-ITQ,he2013a,he2013b}) to minimize the quantization error. The simple algorithm in ITQ \cite{PCA-ITQ} is used for the proposed method.

 The impact of the rotations applied on the learned embeddings is evaluated on MNIST and CIFAR. Figure~\ref{mnist}  clearly shows that the orthogonal rotations achieve significantly performance improvements (in terms of both $F_1$ and recall) for IMH-LE and IMH-tSNE.
In conjunctions with the rotations, the proposed IMH-LE and IMH-tSNE methods  perform much better than the PCA based PCA-ITQ. Again this result demonstrates the advantages of the proposed manifold hashing method.

\begin{algorithm}[t]
\caption{\small Supervised \idh (\IDHs)}
{
\footnotesize
\textbf{Input: }
 Unlabelled data $\BX := \{\Bx_1, \Bx_2, \ldots, \Bx_n\}$, labelled  data $\BXs$ of $t$ classes,
 code length $r$, base set size $m$,
 neighborhood size $k$

\textbf{Output: } Binary codes $\BY := \{\By_1, \By_2, \ldots, \By_n\} \in \mathbb{B}^{n \times r}$

1) Generate the base set $\BB := \{\Bc_{1,1}, \cdots, \Bc_{m_1,1}, \cdots, \Bc_{1,t} \cdots, \Bc_{m_t,t}\}$ by  K-means on data points from each class of $\BXs$.

2) Embed $\BB$ into the low dimensional space by \eqref{EQ:tSNE}, \eqref{EQ:Obj_main}  or any other
appropriate manifold leaning method and get $\BY_{\BB}$.

3) Apply supervised subspace learning algorithms such as LDA on $\BY_{\BB}$.

4) Obtain the low dimensional embedding $\BY$ for the whole dataset inductively by Equation  \eqref{EQ:SubInduction}.

5) Threshold $\BY$ at zero.
}
\label{Alg2}
\end{algorithm}

\section{Semantic hashing with supervised manifold learning}
\label{Sec:supervised}

\begin{figure*}[]
\centering

\includegraphics[height = 5cm]{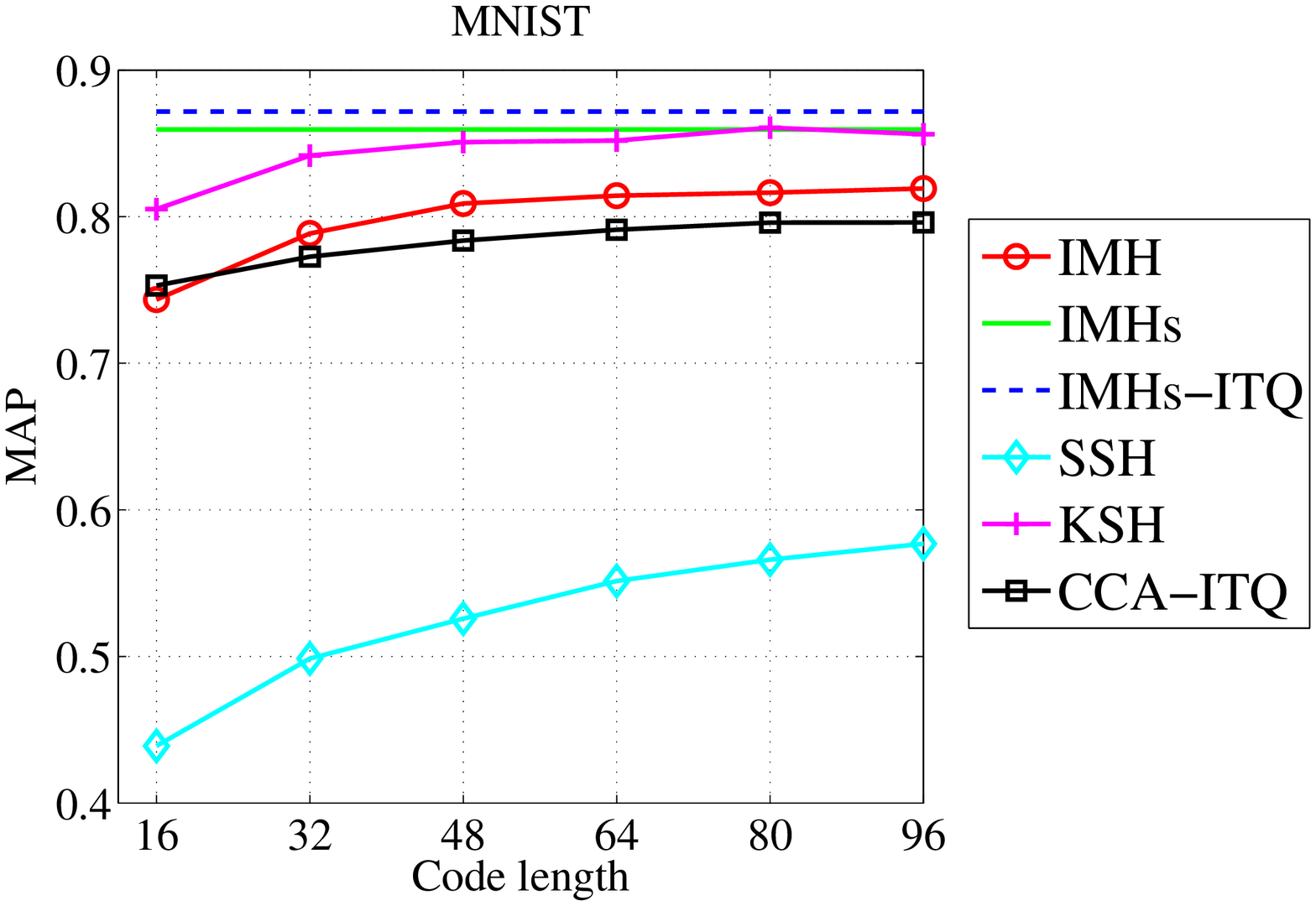}
\includegraphics[height = 5cm]{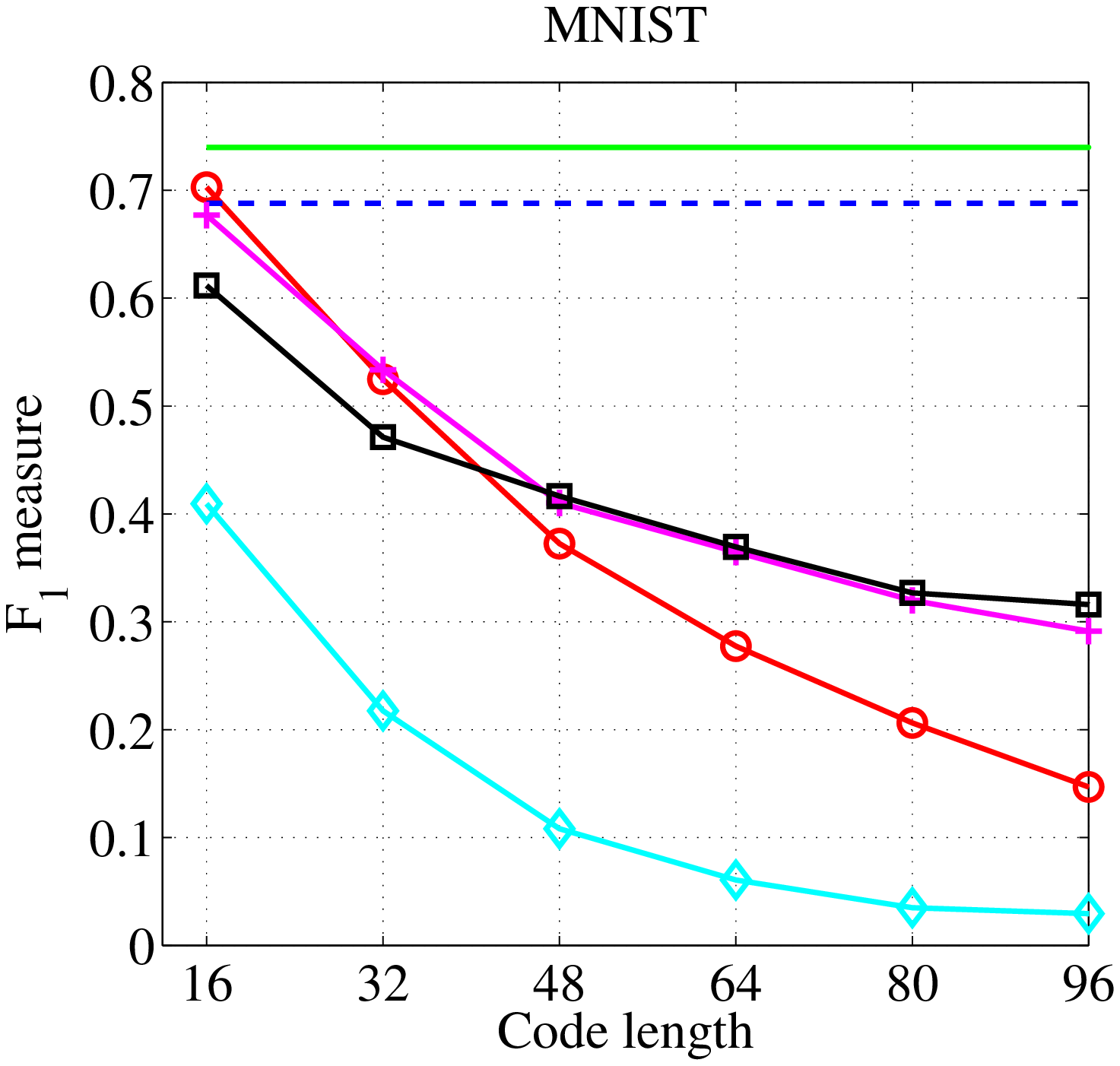}

\includegraphics[height = 5cm]{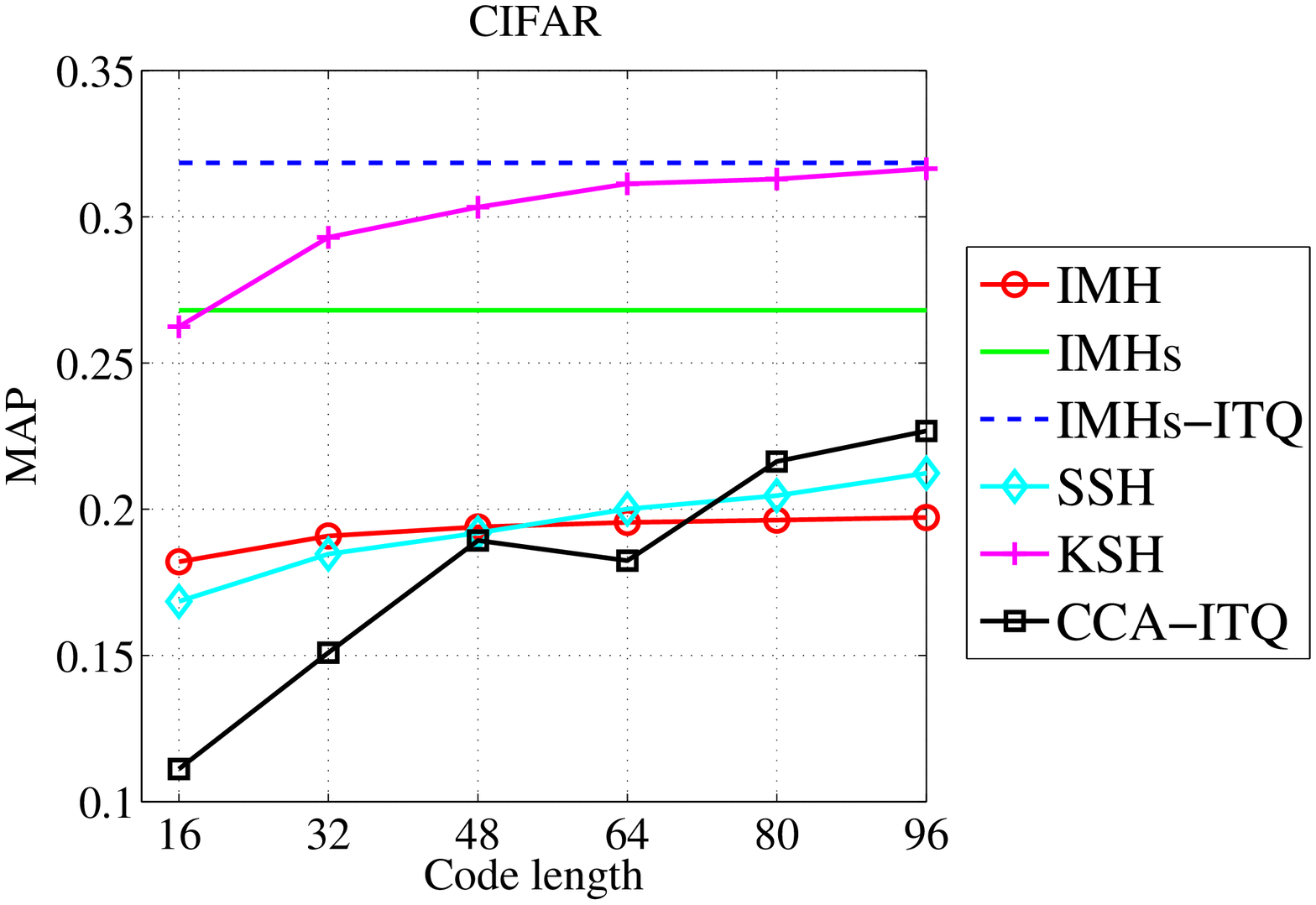}
\includegraphics[height = 5cm]{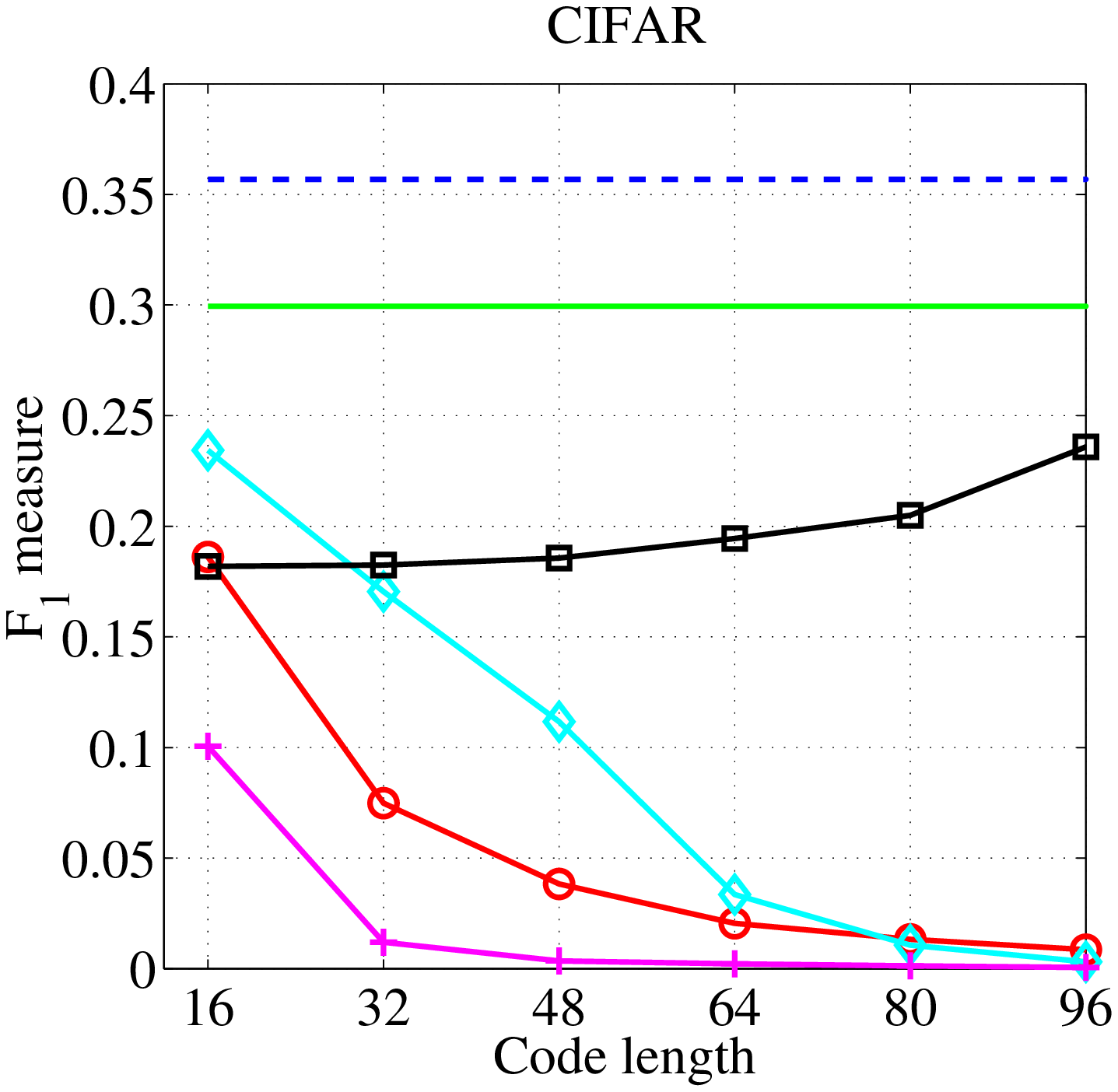}
\caption{Evaluation of IMH with supervised learning  by LDA on the MNIST and CIFAR  datasets. Non-linear embeddings of IMH are obtained by t-SNE.
Since there are only 10 classes with both these two datasets, the reduced dimensionality by LDA (thereby the binary code lenght) is set to 9.}
\label{super}
\end{figure*}

The proposed inductive manifold hashing algorithm has shown to work well on preserving the semantic neighborhood relationships without using label information. It is natural that the performance could be improved by applying supervised learning methods instead of the unsupervised ones to learn the nonlinear embeddings. Unfortunately, most of the manifold learning algorithms are unsupervised.
In order to use the large amount of unsupervised manifold learning methods, a straightforward supervised extension to the proposed IMH algorithm is proposed int this study. First, the base set $\BB := \{\Bc_{1,1}, \cdots, \Bc_{m_1,1}, \cdots, \Bc_{1,t} \cdots, \Bc_{m_t,t}\}$ is generated by applying K-means on data from each of the $t$ classes. After the nonlinear embeddings $\BY_{\BB}$ of $\BB$ are obtained, the supervised subspace learning algorithms are simply conducted on  $\BY_{\BB}$. For a new data point, its binary codes are then obtained by \eqref{EQ:hash_fun}.
The supervised manifold hashing method is summarised in Algorithm \ref{Alg2}.

In this section, the linear discriminant analysis (LDA) is taken as an example to verify the efficacy of the proposed supervised manifold hashing algorithm.
The proposed method is also compared with several recently proposed supervised hashing approaches such as semi-supervised hashing (SSH \cite{SSH2012}) with sequential projection learning, kernel-based supervised hashing (KSH \cite{KSH2012}) and ITQ with Canonical Correlation Analysis (CCA-ITQ \cite{gong2013iterative}).

Experiments are performed on MNIST and CIFAR. In this experiment, 2,000 labelled samples are randomly selected for supervised learning for SSH and KSH, and 1,000 labelled samples are sampled for the base set for IMHs. All labelled training data are used for the linear CCA-ITQ.
Since there are only 10 classes with both these two datasets, the reduced dimensionality by LDA (thereby the binary code length) in the proposed IMHs is fixed to 9. The results are reported in Figure~\ref{super}.
It is clear that the proposed supervised inductive manifold hashing algorithm \IDHs significantly improve the original \IDH  and other compared supervised methods in both MAP and $F_1$ measure.
The ITQ rotations (IMHs-ITQ in the figure) further improves \IDHs with considerable gains, especially on the CIFAR dataset of natural images.
Among other supervised hashing methods, KSH obtains the highest MAPs on MNIST and CIFAR with IMHs or IMHs-ITQ. However it needs  much larger binary code lengths to achieve comparable performance with the proposed methods. In terms of $F_1$, CCA-ITQ is identified to have the best results with long codes.

From these results, it is clear that label information is very useful to achieve semantically effective hashing codes.
And also, the proposed simple supervised hashing framework can effectively leverage the supervised information in the proposed manifold hashing algorithms. Also note that the proposed method does not assume a specific  algorithm like LDA, any other supervised subspace learning or metric learning algorithms may further improve the performance.

\section{Conclusion and discussion}
\label{Sec:conclusion}

The present study has proposed a simple yet effective hashing framework, which provides a practical connection between manifold learning methods (typically non-parametric and with high computational cost) and hash function learning (requiring high efficiency).
By preserving the underlying manifold structure with several non-parametric dimensionality reduction methods,
the proposed
hashing
 methods outperform several state-of-the-art methods
in terms of
both hash lookup and Hamming ranking on several large-scale retrieval-datasets.
The proposed inductive formulation of the hash function sees the suggested methods require only linear time  ($\mathit{O}(n)$) for indexing all of the training data and a constant search time for a novel query.
 Experiments showed that the hash codes can also achieve   promising results on a classification problem even with  very short code lengths.
The proposed inductive manifold hashing method was then extended by applying orthogonal rotations on the learned nonlinear embeddings to minimize the quantization errors, which was shown to achieve significant performance improvements. In addition, this work further extended IMH by adopting supervised subspace learning on the data manifolds, which provided an effective supervised manifold hashing framework.

The proposed hashing methods have been shown to work well on the image retrieval and classification tasks. As an efficient and effective nonlinear feature extraction method, this algorithm can also be applied to other applications, especially the ones need to deal with large dataset. For example, the introduced hashing techniques can be applied to large-scale mobile video retrieval \cite{liu2013near}. Another useful application of the hashing methods would be  compressing the high dimensional features into short binary codes, which could significantly speed up the potential tasks, such as the large scale ImageNet image classification \cite{gong2013learning,CVPR14Lin}.

In this work, the base set size $m$ was set empirically, which was  not an optimal choice apparently. How to automatically set this parameter according to the size and distribution of a specific dataset deserves a future study.

 \section*{Acknowledgements}
 This work was in part supported by ARC Future Fellowship FT120100969.

\bibliographystyle{ieee}
\bibliography{CSRef}

\end{document}